\documentclass{article}

\usepackage[nonatbib, final]{nips_2017}

\usepackage[utf8]{inputenc} % allow utf-8 input
\usepackage[T1]{fontenc}    % use 8-bit T1 fonts
\usepackage{hyperref}       % hyperlinks
\usepackage{url}            % simple URL typesetting
\usepackage{booktabs}       % professional-quality tables
\usepackage{amsfonts}       % blackboard math symbols
\usepackage{nicefrac}       % compact symbols for 1/2, etc.
\usepackage{microtype}      % microtypography
\usepackage{verbatim}
\usepackage{amsmath}
\usepackage[belowskip=-13pt,aboveskip=0pt]{caption}
\usepackage{amsthm}
\usepackage{graphicx} 
\usepackage{amssymb}
\usepackage{placeins}

\theoremstyle{plain}
\newtheorem{thm}{\protect\theoremname}
  \theoremstyle{definition}
  
  \theoremstyle{plain}
  \newtheorem{prop}[thm]{\protect\propositionname}
  \theoremstyle{plain}

\makeatother

  \providecommand{\corollaryname}{Corollary}
  \providecommand{\definitionname}{Definition}
  \providecommand{\propositionname}{Proposition}
\providecommand{\theoremname}{Theorem}

\usepackage{times}
\usepackage{subfigure} 

\usepackage{algorithm}
\usepackage{algorithmic}

\newcommand{\noisE}{{\mathcal{E}}}
\newcommand{\eqd}{{\overset{d}{=}}}

\title{Testing and Learning on Distributions with Symmetric Noise Invariance}

\author{
  Ho Chung Leon Law \\
  Department of Statistics\\
  University Of Oxford\\
  \texttt{hlaw@stats.ox.ac.uk} \\
  \And
  Christopher Yau\\
  Centre for Computational Biology \\
  University of Birmingham \\
  \texttt{c.yau@bham.ac.uk } \\
   \And
  Dino Sejdinovic \\
  Department of Statistics \\
  University Of Oxford \\
  \texttt{dino.sejdinovic@stats.ox.ac.uk} \\
}

\begin{document}

\maketitle

\begin{abstract}
 Kernel embeddings of distributions and the Maximum Mean Discrepancy (MMD), the resulting distance between distributions, are useful tools for fully nonparametric two-sample testing and learning on distributions. However, it is rare that all possible differences between samples are of interest -- discovered differences can be due to different types of measurement noise, data collection artefacts or other irrelevant sources of variability. We propose distances between distributions which encode invariance to additive symmetric noise, aimed at testing whether the assumed true underlying processes differ. Moreover, we construct invariant features of distributions, leading to learning algorithms robust to the impairment of the input distributions with symmetric additive noise. 
%Such features lend themselves to a straightforward neural network implementation and can thus also be learned given a supervised signal.
\end{abstract}

\section{Introduction}
\label{sec: Intro}
%Here review the motivation for phase discrepancies, previous related methods, contributions, and outline of sections, as well as problem set up.
There are many sources of variability in data, and not all of them are pertinent to the questions that a data analyst may be interested in. Consider, for example, a nonparametric two-sample testing problem, which has recently been attracting significant research interest, especially in the context of kernel embeddings of distributions \cite{chwialkowski2015fast, gretton2012kernel, jitkrittum2016interpretable}. We observe samples $\{X_{1j}\}_{j=1}^{N_1}$ and $\{X_{2j}\}_{j=1}^{N_2}$ from two data generating processes $P_1$ and $P_2$, respectively, and would like to test the null hypothesis that $P_1=P_2$ without making any parametric assumptions on these distributions. With a large sample-size, the minutiae of the two data generating processes are uncovered (e.g. slightly different calibration of the data collecting equipment, different numerical precision%, different conventions of dealing with edge-cases across the two processes
), and we ultimately reject the null hypothesis, even if the sources of variation across the two samples may be irrelevant for the analysis. 

Similarly, we may be interested in \emph{learning on distributions} \cite{Muandet2012,Sutherland2016,Szabo2015}, where the appropriate level of granularity in the data is distributional. For example, each label $y_i$ in supervised learning is associated to a whole bag of observations $B_i=\{X_{ij}\}_{j=1}^{N_i}$ -- assumed to come from a probability distribution $P_i$, or we may be interested in clustering such bags of observations. Again, nonparametric distances used in such contexts to facilitate a learning algorithm on distributions, such as Maximum Mean Discrepancy (MMD) \cite{gretton2012kernel}, can be sensitive to irrelevant sources of variation and may lead to suboptimal or even misleading results, in which case building predictors which are invariant to noise is of interest.  
%Moreover, even if the training data is not impaired by measurement noise, the testing data may be, i.e. we may be in a situation of \emph{covariate shift} \cite{Quinonero2009} on distribution inputs, and again 

While it may be tempting to revert back to a parametric setup and work with simple, easy to interpret models, we argue that a different approach is possible: we stay within a nonparametric framework, exploit the irregular and complicated nature of real life distributions and \emph{encode invariances} to sources of variation assumed to be irrelevant. In this contribution, we focus on \emph{invariances to symmetric additive noise} on each of the data generating distributions. Namely, assume that the $i$-th sample $\{X_{ij}\}_{j=1}^{N_i}$ we observe does not follow the distribution $P_i$ of interest but instead its convolution $P_i\star \noisE_i$ with some unknown noise distributions $\noisE_i$ assumed to be symmetric about $0$ (we also require that it has a positive characteristic function). We would like to assess the differences between $P_i$ and $P_{i'}$ while allowing $\noisE_{i}$ and $\noisE_{i'}$ to differ in an arbitrary way. We investigate two approaches to this problem: (1) measuring the degree of asymmetry of the paired differences $\{X_{ij}-X_{i'j}\}$, and (2) comparing the \emph{phase functions} of the corresponding samples. While the first approach is simpler and presents a sensible solution for the two-sample testing problem, we demonstrate that phase functions give a much better gauge on the \emph{relative comparisons} between bags of observations, as required for learning on distributions. %We construct a neural network formulation of the latter approach, allowing backpropagation for learning discriminative distribution features based on phase functions.  

The paper is outlined as follows. In section \ref{sec:setup}, we provide an overview of the background. In section \ref{sec:main}, we provide details of the construction and implementation of phase features. In section \ref{sec:paired differences}, we discuss the approach based on asymmetry in paired differences for two sample testing with invariances. Section \ref{sec:results} provides experiments on synthetic and real data, before concluding in section \ref{sec: conclu}. 
\section{Background and Setup}
\label{sec:setup}
We will say that a random vector $E$ on $\mathbb R^d$ is a \emph{symmetric positive definite (SPD) component} if its characteristic function is positive, i.e. $\varphi_E(\omega)=\mathbb E_{X\sim E} \left[\exp(i\omega^\top E)\right]>0$, $\forall \omega\in\mathbb R^d$. This means that $E$ is (1) symmetric about zero, i.e. $E$ and $-E$ have the same distribution and (2) if it has a density, this density must be a positive definite function \cite{rossberg1995positive}. Note that many distributions used to model additive noise, including the spherical zero-mean Gaussian distribution, as well as multivariate Laplace, Cauchy or Student's $t$ (but not uniform), are all SPD components.

Following the terminology similar to that of \cite{delaigle2016methodology}, we will say that a random vector $X$ on $\mathbb R^d$ is \emph{decomposable} if its characteristic function can be written as $\varphi_X=\varphi_{X_0}\varphi_{E}$, with $\varphi_{E}>0$. Thus, if $X$ can be written in the form $X=X_0+E$, where $X_0$ and $E$ are independent and $E$ is an SPD noise component, then $X$ is decomposable. We will say that $X$ is \emph{indecomposable} if it is not decomposable.  In this paper, we will assume that mostly the indecomposable components of distributions are of interest and will construct tools to directly measure differences between these indecomposable components, encoding invariance to other sources of variability. The class of Borel Probability measures on $\mathbb R^d$ will be denoted $\mathcal{M}^{1}_+(\mathbb{R}^d)$, while the class of indecomposable probability measures will be denoted by $\mathcal{I}(\mathbb{R}^d) \subseteq \mathcal{M}^{1}_+(\mathbb{R}^d)$.    
\subsection{Kernel Embeddings, Fourier Features and learning on distributions} 
\label{sec:background}
For any positive definite function $k\colon \mathcal{X} \times \mathcal{X} \mapsto \mathbb{R}$, there exists a unique reproducing kernel Hilbert space (RKHS) $\mathcal H_k$ of real-valued functions on $\mathcal{X}$. Function $k(\cdot,x)$ is an element of $\mathcal H_k$ and represents evaluation at $x$, i.e. $\langle f, k(\cdot,x)\rangle_\mathcal{H} = f(x)$, $\forall f\in\mathcal H_k$, $\forall x\in\mathcal X$. The kernel mean embedding (cf. \cite{muandet2016kernel} for a recent review) of a probability measure $P$ is defined by
$\mu_P = \mathbb{E}_{X\sim P}[k(\cdot,X)] = \int_{\mathcal X} k(\cdot, x)dP(x)$. The Maximum Mean Discrepancy (MMD) between probability measures $P$ and $Q$ is then given by $\Vert \mu_P-\mu_Q\Vert_{\mathcal H_k}$. For shift-invariant kernels on $\mathbb R^d$, using Bochner's characterisation of positive definiteness \cite[6.2]{Wendland2004}, the squared MMD can be written as a weighted $L_2$-distance between characteristic functions \cite[Corollary 4]{Sriperumbudur2010}
\begin{equation}
\Vert \mu_P-\mu_Q\Vert^2_{\mathcal H_k}=\int_{\mathbb R^d}\left|\varphi_{P}\left(\omega\right)-\varphi_{Q}\left(\omega\right)\right|^{2}d\Lambda\left(\omega\right),
\end{equation}
where $\Lambda$ is the non-negative spectral measure (inverse Fourier transform) of kernel $k$ as a function of $x-y$, while $\varphi_P(\omega)$ and $\varphi_Q(\omega)$ are the characteristic functions of probability measures $P$ and $Q$. 

Bochner's theorem is also used to construct random Fourier features (RFF) \cite{rahimi2007random} for fast approximations to kernel methods in order to approximate a pre-specified shift-invariant kernel by a finite dimensional explicit feature map. If we can draw samples from its spectral measure $\Lambda$, we can approximate $k$ by\footnote{a \emph{complex feature map}  \scalebox{0.87}{$\phi(x)=\sqrt{\frac{1}{m}}\left[\exp\left(i\omega_1^\top x\right),\ldots,\exp\left(i\omega_m^\top x\right)\right]$} can also be used, but we follow the convention of real-valued Fourier features, since kernels of interest are typically real-valued.} 
$$ \hat{k}(x,y)  = % \scalebox{0.85}{  
 \frac{1}{m}\sum_{j=1}^m \big[\cos(\omega_j^T x)\cos(\omega_j^T y) + \sin (\omega_j^T x) \sin (\omega_j^T y)\big]  =  \langle \phi(x), \phi(y) \rangle_{\mathbb R^{2m}}  %}
$$where $\omega_1,\dots,\omega_m \sim \Lambda$ and
 \scalebox{0.85}{ 
 $ \phi(x):=\sqrt{\frac{1}{m}}\left[\cos\left(\omega_1^\top x\right),\sin\left(\omega_1^\top x\right)\ldots,\cos\left(\omega_m^\top x\right),\sin\left(\omega_m^\top x\right)\right].  $}
Thus, the explicit computation of the kernel matrix is not needed and the computational complexity is reduced. This also allows computation with the approximate, finite-dimensional embeddings $\tilde\mu_P=\Phi(P)={\mathbb E}_{X\sim P} \phi(X)\in\mathbb R^{2m}$, which can be understood as the evaluations (real and complex part stacked together) of the characteristic function $\varphi_P$ at frequencies $\omega_1,\dots,\omega_m$. We will refer to the approximate embeddings $\Phi(P)$ as Fourier features of distribution $P$.
%\subsection{Learning on Distributions}
%\label{sec:dr}

Kernel embeddings can be used for supervised learning on distributions. Assume we have a training set $\{B_i, y_i\}_{i = 1}^{n}$, where input $B_i=\{x_{ij}\}_{j=1}^{N_i}$ is a bag of samples taking values in $\mathcal X$, and $y_i$ is a response. Given a kernel $k\colon \mathcal{X} \times \mathcal{X} \to \mathbb{R}$, we first map  each $B_i$ to the empirical embedding $\mu_{\hat P_i}=\frac{1}{N_i}\sum_{j=1}^{N_i}k(\cdot,x_{ij})\in\mathcal{H}_k$ and then can apply any positive definite kernel on $\mathcal H_k$ as the kernel on bag inputs, e.g. linear kernel $\tilde{K}(B_i,B_i')=\langle \mu_{\hat P_i}, \mu_{\hat P_{i'}} \rangle_{\mathcal H_k}$, in order to perform classification \cite{Muandet2012} or regression \cite{Szabo2015}. Approximate kernel embeddings have also been applied in this context \cite{Sutherland2016}. %Distribution regression has recently been applied to Approximate Bayesian Computation (ABC) in order to construct optimal summary statistics for posterior inference of model parameters \cite{mitrovic2016dr}. In the Appendix, we discuss a similar application to ABC of invariant distribution representations developed in this paper.  
\section{Phase Discrepancy and Phase Features}
\label{sec:main}
While MMD and kernel embeddings are related to characteristic functions, and indeed the same connection forms a basis for fast approximations to kernel methods using random Fourier features  \cite{rahimi2007random}, the relevant notion in our context is the \emph{phase function} of a probability measure, recently used for nonparametric deconvolution by \cite{delaigle2016methodology}. In this section, we overview this formalism. Based on the empirical phase functions, we will then derive and investigate hypothesis testing and learning framework using \emph{phase features of distributions}.
 
In nonparametric deconvolution \cite{delaigle2016methodology}, the goal is to estimate the density function $f_{0}$ of a univariate r.v. $X_0$, but in general we only have noisy data samples $X_1, \dots, X_n 
\stackrel{iid}{\sim} X = X_0 + E$, where $E$ denotes an independent noise term. %Without obtaining some more information about $U$ or making some assumptions on the underlying distributions, this is clearly a difficult problem.
Even though the distribution of $E$ is unknown, making the assumption that $E$ is an SPD noise component, and that $X_0$ is indecomposable, i.e. $X_0$ itself does not contain any SPD noise components, \cite{delaigle2016methodology} show that it is possible to obtain consistent estimates of $f_0$. 

They distinguish between the symmetric noise and the underlying indecomposable component by matching phase functions, defined as
$$\rho_{X}\left(\omega\right)=\frac{\varphi_{X}\left(\omega\right)}{\left|\varphi_{X}\left(\omega\right)\right|}%=\exp\left(i\tau_{X}\left(\omega\right)\right)
$$where $\varphi_{X}\left(\omega\right)$ denotes the characteristic function of $X$. 
%We follow the terminology of \cite{delaigle2016methodology} but note that $\tau_{X}$ may also be called the phase function. 
Observe that $\left|\rho_{X}\left(\omega\right)\right|=1$, and thus we are effectively removing the amplitude information from the characteristic function. For a SPD noise component $E$, the phase function is $\rho_E(\omega)\equiv 1$. But then since $\varphi_{X} = \varphi_{X_0} \varphi_{E}$, we have that $\rho_{X_{0}} = \rho_{X} = \varphi_{X} / |\varphi_{X} |$, i.e. the phase function is invariant to additive SPD noise components. This motivates us to construct explicit feature maps of distributions with the same property and similarly to the motivation of \cite{delaigle2016methodology}, we argue that real-world distributions of interest often exhibit certain amount of irregularity and it is exactly this irregularity which is exploited in our methodology.

In analogy to the MMD, we first define the phase discrepancy (PhD) as a weighted $L_{2}$-distances between the phase functions:
\begin{equation}
\label{eqn:phd}
\text{PhD}(X,Y) =\displaystyle\int_{\mathbb R^d}\left|\rho_{X}\left(\omega\right)-\rho_{Y}\left(\omega\right)\right|^{2}d\Lambda\left(\omega\right) 
\end{equation}
for some non-negative measure $\Lambda$ (w.l.o.g. a probability measure). Now suppose we write $X=X_{0}+U$, $Y=Y_{0}+V$, where $U$ and $V$ are SPD noise components.  This then implies $\rho_{X}=\rho_{X_{0}}$ and $\rho_{Y}=\rho_{Y_{0}}$ $\Lambda$-everywhere, so that $\text{PhD}(X,Y)=\text{PhD}(X_{0},Y_{0})$. It is clear then that the PhD is not affected by additive SPD noise components, so it captures desired invariance. However, the PhD for $\Lambda$ supported everywhere is in fact not a proper metric on the indecomposable probability measures $\mathcal I(\mathbb{R}^d)$, as one can find indecomposable random variables $X$ and $Y$ s.t. $\rho_X=\rho_Y$ and thus $\text{PhD}(X,Y)=0$. An example is given in Appendix \ref{sec:indecom}.

While such cases appear contrived, we hence restrict attention to a subset of indecomposable probability measures $\mathcal{P}(\mathbb{R}^d) \subset \mathcal{I}(\mathbb{R}^d)$, which are uniquely determined by phase functions, i.e. $\forall P, Q \in \mathcal{P}(\mathbb{R}^d): \rho_P = \rho_Q \Rightarrow P = Q$.

We now have the two following propositions (proofs are given in Appendix \ref{sec:proofs}).
\begin{prop} 
$$
\scalebox{1.0}{
 \text{PhD}$(X,Y)=2-2\int \left( \frac{\mathbb{E}\xi_{\omega}(X)}{\left\Vert \mathbb{E}\xi_{\omega}(X)\right\Vert }\right)^{\top} \left( \frac{\mathbb{E}\xi_{\omega}(Y)} {
\left\Vert \mathbb{E} \xi_{\omega} (Y) \right\Vert } \right) d\Lambda(\omega) $} $$
where $\xi_{\omega}\left(x\right)=\left[\cos\left(\omega^{\top}x\right),\sin\left(\omega^{\top}x\right)\right]^\top$ and $\left\Vert \  \cdot  \ \right\Vert$ denotes the standard $L_2$ norm.
\end{prop} 
\begin{prop}
$$
\scalebox{1.0}{$K\left(P_{X},P_{Y}\right)=\int\left(\frac{\mathbb{E}\xi_{\omega}(X)}{\left\Vert \mathbb{E}\xi_{\omega}(X)\right\Vert }\right)^{\top}\left(\frac{\mathbb{E}\xi_{\omega}(Y)}{\left\Vert \mathbb{E}\xi_{\omega}(Y)\right\Vert }\right)d\Lambda(\omega)$}$$
is a positive definite kernel on probability measures.
\end{prop}
Now, we can construct an approximate explicit feature map for kernel $K$. Taking a sample $\left\{ \omega_{i}\right\} _{i=1}^{m}\sim\Lambda$, we define $\Psi:P_{X}\mapsto\mathbb{R}^{2m}$ given by
$\scalebox{1.0}{$\Psi(P_{X}) = \sqrt{\frac{1}{m}} \left[\frac{\mathbb{E}\xi_{\omega_{1}}(X)} {\left\Vert \mathbb{E}\xi_{\omega_{1}}(X)\right\Vert} ,\ldots,\frac{\mathbb{E}\xi_{\omega_{m}}(X)}{\left\Vert \mathbb{E}\xi_{\omega_{m}}(X)\right\Vert }\right] $}
$. We will refer to $\Psi(\cdot)$ as the \emph{phase features}. Note that these are very similar to Fourier features, but the $\cos,\sin$-pair corresponding to each frequency is normalised to have unit $L_2$ norm. In other words, $\Psi(\cdot)$ can be thought of as evaluations of the phase function at the selected frequencies. By construction, phase features are invariant to additive SPD noise components. For an empirical measure, we simply have the following: \begin{equation}
\label{eqn:phase_em}
\scalebox{1.0}{$ \Psi(\hat{P}_{X})=\sqrt{\frac{1}{m}} \left[\frac{\hat{\mathbb{E}}\xi_{\omega_{1}}(X)}{\left\Vert \mathbb{\hat{E}}\xi_{\omega_{1}}(X)\right\Vert },\ldots,\frac{\mathbb{\hat{E}}\xi_{\omega_{m}}(X)}{\left\Vert \mathbb{\hat{E}}\xi_{\omega_{m}}(X)\right\Vert }\right]$}
\end{equation}
where we have replaced the expectations by their empirical estimates. Because $\left\Vert \Psi(\hat{P}_{X})\right\Vert =1$, we can construct
\begin{equation}
\label{eqn:phase_statistic}
\widehat{\text{PhD}}(\hat{P}_{X},\hat{P}_{Y}) =  \left\Vert \Psi(\hat{P}_{X})-\Psi(\hat{P}_{Y})\right\Vert ^{2}
 =  2-2\Psi(\hat{P}_{X})^{\top}\Psi(\hat{P}_{Y}),
\end{equation}
which is a Monte Carlo estimator of $\text{PhD}(\hat{P}_{X},\hat{P}_{Y})$. In summary, $\Psi(\hat{P})\in\mathbb R^{2m}$ is an explicit feature vector of the empirical distribution which encodes invariance to additive SPD noise components present in $P$ \footnote{Note that, unlike the population expression $\Psi(P)$, the empirical estimator $\Psi(\hat P)$ will in general have a distribution affected by the noise components and is thus only approximately invariant, but we observe that it captures invariance very well as long as the signal-to-noise regime remains relatively high (Section \ref{sec: sample}).}, as demonstrated in Figure \ref{sec:pdexp} in the Appendix. It can now be directly applied to (1) two-sample testing up to SPD components, where the distance between the phase features, i.e. an estimate \eqref{eqn:phase_statistic} of the PhD, can be used as a test statistic, with details given in section \ref{sec: sample} and (2) learning on distributions, where we use phase features as the explicit feature map for a bag of samples. 

Although we have assumed an indecomposable underlying distribution so far, this assumption is not strict. For distribution regression, if the indecomposable assumption is invalid, given that the underlying distribution is irregular, it may still be useful to encode invariance as long as the benefit of removing the SPD components irrelevant for learning outweighs the signal in the SPD part of the distribution, i.e. there is a trade off between SPD noise and SPD signal. In practice, the phase features we propose can be used to encode such invariance where appropriate or in conjunction with other features which do not encode invariance. %thus letting the data decide which features are discriminative for the problem at hand.

In order to construct the approximate mean embeddings for learning, we first compute an explicit feature map by taking averages of the Fourier features, as given by 
$ \Phi(\hat{P}_X) = \sqrt{\frac{1}{m}}\left[\hat{\mathbb{E}}\xi_{\omega_{1}}(X),\ldots,\mathbb{\hat{E}}\xi_{\omega_{m}}(X)\right]$. For phase features, we need to compute an additional normalisation term over each frequency as in (\ref{eqn:phase_em}). To obtain the set of frequencies $\{w_i\}^m_{i = 1}$, we can draw samples from a probability measure $\Lambda$ corresponding to an inverse Fourier transform of a shift-invariant kernel, e.g. Gaussian Kernel. However, given a supervised signal, we can also optimise a set of frequencies $\{w_i\}^m_{i = 1}$ that will give us a useful representation and good discriminative performance. In other words, we no longer focus on a specific shift-invariant kernel $k$, but are \emph{learning discriminative Fourier/phase features}. To do this, we can construct a neural network (NN) with special activation functions, pooling layers as shown in Algorithm \ref{alg:nn} and Figure \ref{fig: pnn} in the Appendix.
\section{Asymmetry in Paired Differences}
\label{sec:paired differences}
We now consider a separate approach to nonparametric two-sample test, where we wish to test the null hypothesis that $H_0: P \eqd Q$ vs. the general alternative, but we only have iid samples arising from $X\sim P\star \noisE_1$ and $Y\sim Q\star \noisE_2$. i.e.
\begin{eqnarray*}
X = X_0 + U &
Y= Y_0 + V 
\end{eqnarray*}
where $X_0\sim P$, $Y_0\sim Q$ lie in the space of $\mathcal{P}(\mathbb{R}^d)$ of indecomposable distributions uniquely determined by phase functions and $U$ and $V$ are SPD noise components. With this setting (proof in Appendix \ref{sec:proofs}):
\begin{prop} 
Under the null hypothesis $H_0$, $X - Y \text{ is SPD } \iff X_0 \eqd Y_0.$
\end{prop} 

This motivates us to simply perform a two-sample test on $X -Y$ and $Y - X$ since its rejection would imply rejection of $X_0 \eqd Y_0$, as it tests for symmetry. However, note that this is a test for symmetry only and that for consistency against all alternatives, positivity of characteristic function would need to be checked separately. Now, given two i.i.d. samples $\{X_i\}_{i=1}^{n}$ and $\{Y_i\}_{i=1}^{n}$ with $n$ even, we split the two samples into two halves and compute $W_{i}=X_{i}-Y_{i}$ on one half and $Z_{i}=Y_{i}-X_{i}$ on the other half, and perform a nonparametric two sample test on $W$ and $Z$ (which are, by construction, independent of each other). The advantage of this regime is that we can use any two-sample test -- in particular in this paper, we will focus on the \textit{linear time} mean embedding (ME) test \cite{jitkrittum2016interpretable}, which was found to have performance similar to or better than the original MMD two-sample test \cite{gretton2012kernel}, and explicitly formulates a criterion which maximises the test power. We will refer to the resulting test on paired differences as the Symmetric Mean Embedding (SME). 

Although we have assumed here that $X_0$, $Y_0$ lie in the space $\mathcal{P}(\mathbb{R}^d)$ of indecomposable distributions, in practice, the SME test would not reject if the underlying distributions of interest \textit{differ only in the symmetric components} (or in the SPD components for the PhD test). We argue this to be unlikely due to real life distributions being complex in nature with interesting differences often having a degree of asymmetry. In practice, we recommend the use of the ME and SME or PhD test together to provide an exploratory tool to understand the underlying differences, as demonstrated in the Higgs Data experiment in section \ref{sec: sample}.

It is tempting to also consider learning on distributions with invariances using this formalism. However note that the MMD on paired differences is \emph{not invariant to the additive SPD noise components} under the alternative, i.e. in general
$\text{MMD}(X-Y,Y-X)\neq \text{MMD}(X_0-Y_0,Y_0-X_0).$
This means that the paired differences approach to learning is sensitive to the actual type and scale of the additive SPD noise components, hence not suitable for learning. The mathematical details and empirical experiments to show this are presented in Appendix \ref{sec:paried_append} and \ref{sec:pdexp}.
\section{Experimental Results}
\label{sec:results}
\subsection{Two-Sample Tests with Invariances}
\label{sec: sample}
In this section, we demonstrate the performance of the SME test and the PhD test on both artificial and real-world data for testing the hypothesis $H_0: X_0\eqd Y_0$ based on samples $\{X_i\}_{i=1}^{N}$ from $X_0+U$ and $\{Y_i\}_{i=1}^{N}$ from $Y_0+V$, where $U$ and $V$ are arbitrary SPD noise components (we assume the same number of samples for simplicity). SME test follows the setup in \cite{jitkrittum2016interpretable} but applied to $\{X_i-Y_i\}_{i=1}^{N/2}$ and $\{Y_i-X_i\}_{i=N/2+1}^{N}$. For the PhD test, we use as the test statistic the estimate $\widehat{\text{PhD}}(\hat{P}_{X},\hat{P}_{Y})$ of (\ref{eqn:phd}). It is unclear what the exact form of the null distribution is, so we use a permutation test, by recomputing this statistic on the samples which are first merged and then randomly split in the original proportions.  While we are combining samples with different distributions, the permutation test is still justified since, under the null hypothesis $X_0\eqd Y_0$, the resulting characteristic function $\varphi_{null}$ of the mixture can be written as 
$$\varphi_{null} = \frac 1 2\varphi_{X_0} \varphi_U + \frac 1 2 \varphi_{X_0} \varphi_V = \varphi_{X_0}  (\frac 1 2 \varphi_{U} + \frac 1 2 \varphi_{V} )$$ and since the mixture of the SPD noise terms is also SPD, we have that $\rho_{null} = \rho_{X_0} = \rho_{Y_0}$. For our experiments, we denote by $N$ the sample size, $d$ the dimension of the samples, and we take $\alpha=0.05$ to be the significance level. In the SME test, we take the number of test locations $J$ to be $10$, and use $20\%$ of the samples to optimise the test locations. All experimental results are averaged over $1000$ runs, where each run repeats the simulation or randomly samples without replacement from the dataset.
\begin{figure}[t]
    \centering
    \begin{subfigure}
        \centering
        \includegraphics[width=0.49\textwidth]{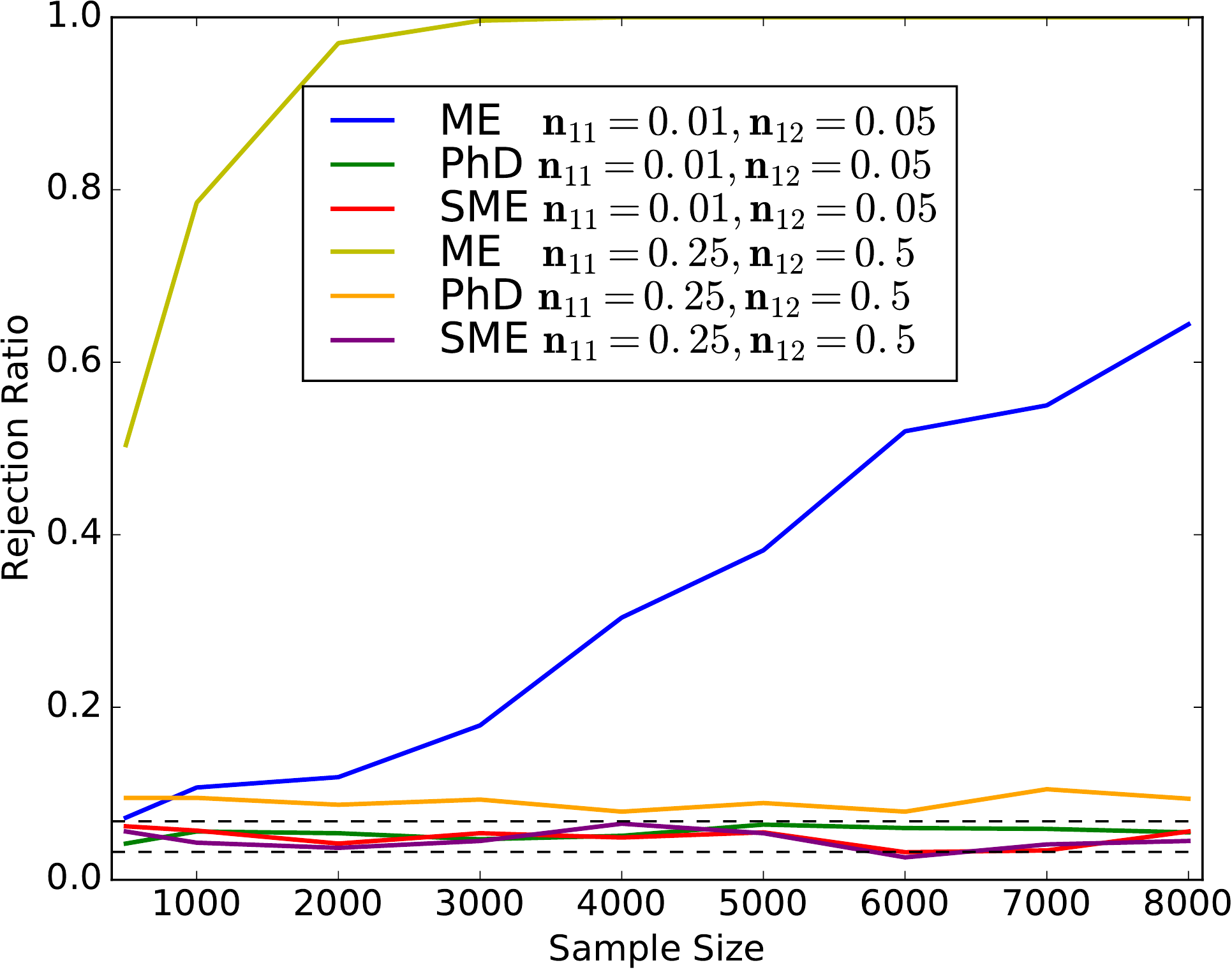}
    \end{subfigure}%
    \hspace{-0.25cm}
    ~ 
    \begin{subfigure}
        \centering
        \includegraphics[width=0.49\textwidth]{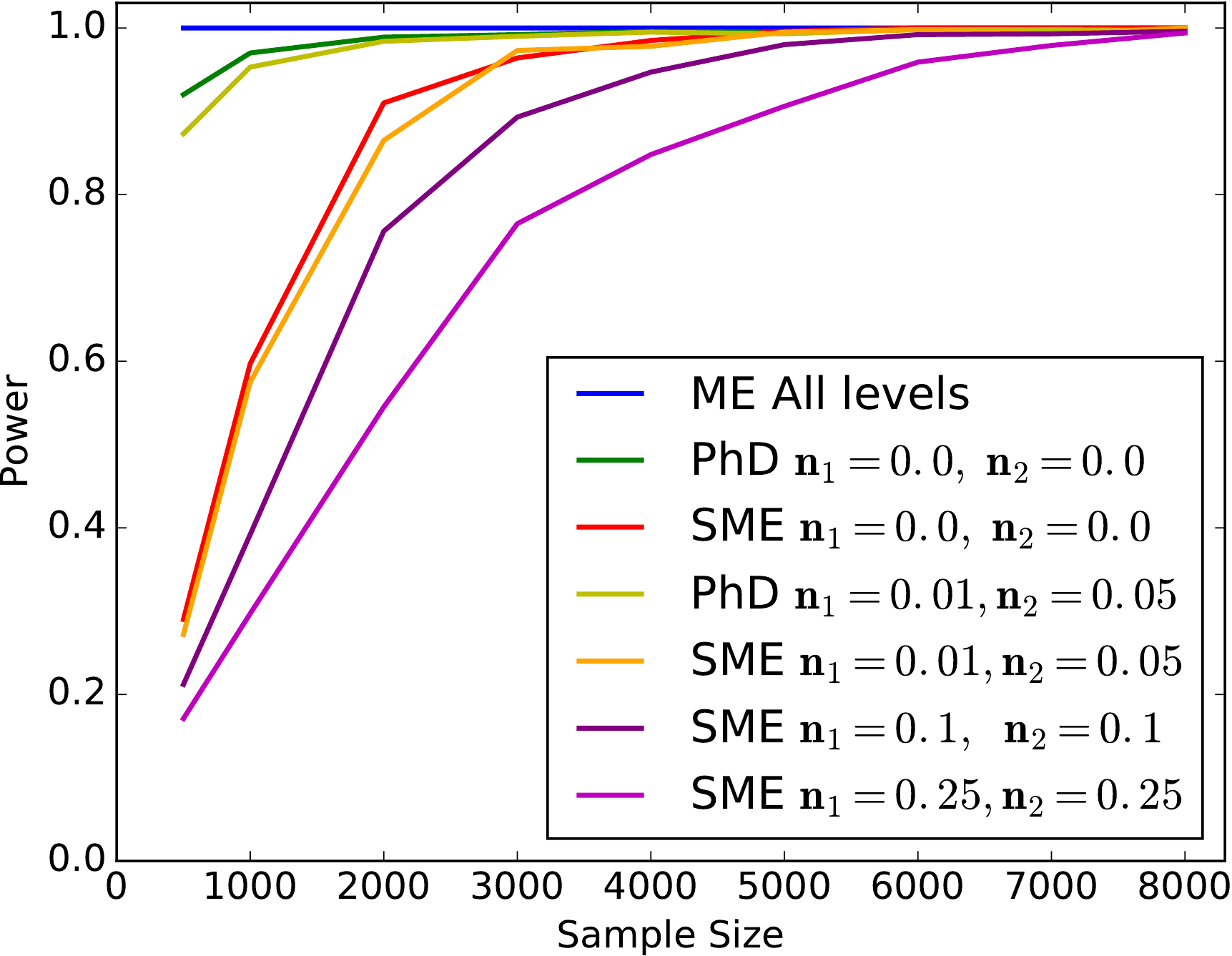}
    \end{subfigure}
    \caption{Type I error and Power under various additional symmetric noise in the synthetic $\chi^2$ dataset. Dashed line is the $99\%$ Wald interval here. \textbf{Left:} Type I error, $n_{11}$ denotes the noise to signal ratio for the first set of samples and $n_{12}$ for the second set. \textbf{Right:} Power, $n_{1}$ denotes the noise to signal ratio for the $X$ set of samples and $n_{2}$ denotes the noise to signal ratio for the $Y$ set of samples.}
\label{fig:chi}
\end{figure}

\subsubsection{Synthetic example: Noisy $\chi^2$}   We start by demonstrating our tests with invariances on a simulated dataset where $X_0$ and $Y_0$ are random vectors with $d=5$, each dimension is the same in distribution and follows $\chi^{2}(4)/4$ and $\chi^{2}(8)/8$ respectively, i.e. chi-squared random variables, with different degrees of freedom, rescaled to have the same mean $1$ (but have different variances, $1/2$ and $1/4$ respectively). An illustration of the true and empirical phase and characteristic function with noise for these two distributions can be found in Appendix \ref{sec:char_phase}. We construct samples $\{X_{n_1,i}\}_{i=1}^N$ and $\{Y_{n_2,i}\}_{i=1}^N$ such that $X_{n_1} \sim X_0 + U$, where $U \sim \mathcal{N}(0,\sigma_1^2 I)$ and similarly $Y_{n_2} \sim Y_0 + V$, where $V \sim \mathcal{N}(0,\sigma_2^2 I)$,  $n_i$ denotes the noise-to-signal ratio given by  the ratio of variances in each dimension, i.e. $n_1 = 2\sigma_1^2$ and $ n_2 = 4\sigma_2^2$. 

We first verify that Type I error is indeed controlled at our design level of $\alpha = 0.05$ \textit{up to various additive SPD noise components}. This is shown in Figure \ref{fig:chi} (left), where $X_0\eqd Y_0$, both constructed using $\chi^{2}(4)/4$, with the noiseless case found in Figure \ref{fig:phd} in the Appendix. It is noted here that the ME test rejects the null hypothesis for even a small difference in noise levels, hence it is unable to let us \textit{target the underlying distributions} we are concerned with. This is unlike the SME test which controls the Type I error even for large differences in noise levels. 
The PhD test, on the other hand, while correctly controlling Type I at small noise levels, was found to have inflated Type I error rates for large noise, with more results and explanation provided in Figure \ref{fig:phd} in the Appendix. %This is due to the sensitivity to noise in the permutation test. 
Namely, the test relies on the invariance to SPD of the population expression of PhD, but the estimator of the null distribution of the corresponding test statistic will in general be affected by the differing noise levels.

Next, we investigate the power, shown in Figure \ref{fig:chi} (right).
%Even at large noise levels, the SME test discovers the difference between the underlying $\chi^2$-distributions given a sufficient sample size. 
For a fair comparison, we have included the PhD test power only for small noise levels, in which the Type I error is controlled at the design level. In these cases, the PhD test has better power than the SME test. This is not surprising, as for the SME we have to halve the sample size in order to construct a valid test. However, recall that the PhD test has an inflated Type I error for large noises, which means that its results should be considered with caution in practice. ME test rejects at all levels at all sample sizes as it picks up all possible differences. SME and PhD are by construction more conservative tests whose rejection provides a much stronger statement: two samples differ even when \textit{all arbitrary additive SPD components} have been stripped off. %In practice, we recommend using both the SME and ME test together, as this can provide insights about the data generating processes at hand. We demonstrate an example of this next. 

\subsubsection{Higgs Dataset}   The UCI Higgs dataset \cite{baldi2014searching, Lichman:2013} is a dataset with $11$ million observations, where the problem is to distinguish between the signal process where Higgs bosons are found, versus the background process that do not produce Higgs bosons. In particular, we will consider a two-sample test with the ME and SME test on the high level features derived by physicists, as well as a two-sample test on four extremely low level features (azimuthal angular momentum $\phi$ measured by four particle jets in the detector). The high level features here (in $\mathbb{R}^{7}$) have been shown to have good discriminative properties in \cite{baldi2014searching}. Thus, we expect them to have different distributions across two processes. Denoting by $X$ the high level features of the process without Higgs Boson, and $Y$ as the corresponding distribution for the processes where Higgs bosons are produced, we test the null hypothesis that the indecomposable parts of $X$ and $Y$ agree. The results can be found in Table \ref{tab:higgs} in the Appendix, which shows that the high level features differ even up to additive SPD components, with a high power for the SME and ME test even at small sample sizes (rejection rate of $0.94$ at $N=500$). Now we perform the same experiment, but with the low level features $\in \mathbb{R}^4$, commented in \cite{baldi2014searching} to carry very little
discriminating information, using the setup from \cite{chwialkowski2015fast}. 

The results for the ME and SME test can be found in Figure \ref{fig: higgs_low}. Here we observe that while ME test clearly rejects and finds the difference between the two distributions,  there is no evidence that the indecomposable parts of the joint distributions of the angular momentum actually differ. In fact, the test rejection rate remains around the chosen design level of $\alpha = 0.05$ for all sample sizes. This highlights the significance in using the SME test, suggesting that the nature of the difference between the two processes can potentially be explained by some additive symmetric noise components which may be irrelevant for discrimination, providing an insight into the dataset. Furthermore, this also highlights the argument that given two samples from complex data collection and generation processes, a nonparametric two sample test like ME will likely reject given sufficient sample sizes, even if the discovered difference may not be of interest. With the SME test however, we can ask a much more subtle question about the differences between the assumed true underlying processes. Figures showing that the Type I error is controlled at the design level of $\alpha = 0.05$ for both low and high level features can be found in Figure \ref{fig:higgs_type1} in the Appendix.
\subsection{Learning with Phase Features}
\label{sec:learn}
\begin{figure}[t]
\centering
\begin{minipage}{.475\textwidth}
  \centering
\includegraphics[scale=0.36]{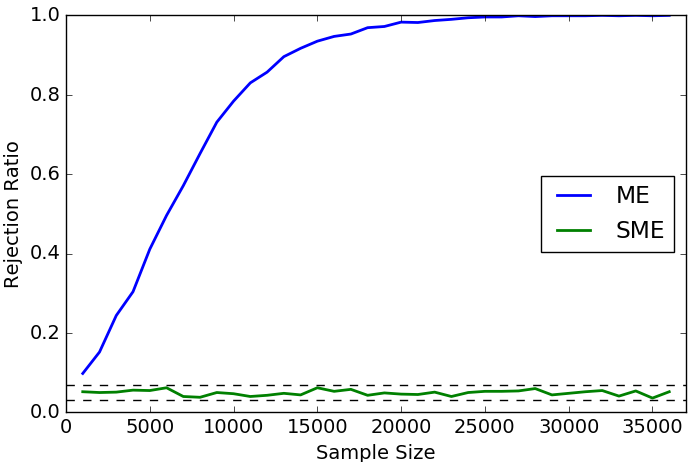}
  \vspace{0.5cm}
  \captionof{figure}{Rejection ratio vs. sample size for extremely low level features for Higgs dataset. Dashed line is the $99\%$ Wald interval for 1000 repetitions for $\alpha =0.05$. Note PhD is not used here, due to its expensive computational cost.}
  \label{fig: higgs_low}
\end{minipage}%
\hspace{0.2cm}
\begin{minipage}{.475\textwidth}
  \includegraphics[scale = 0.245]{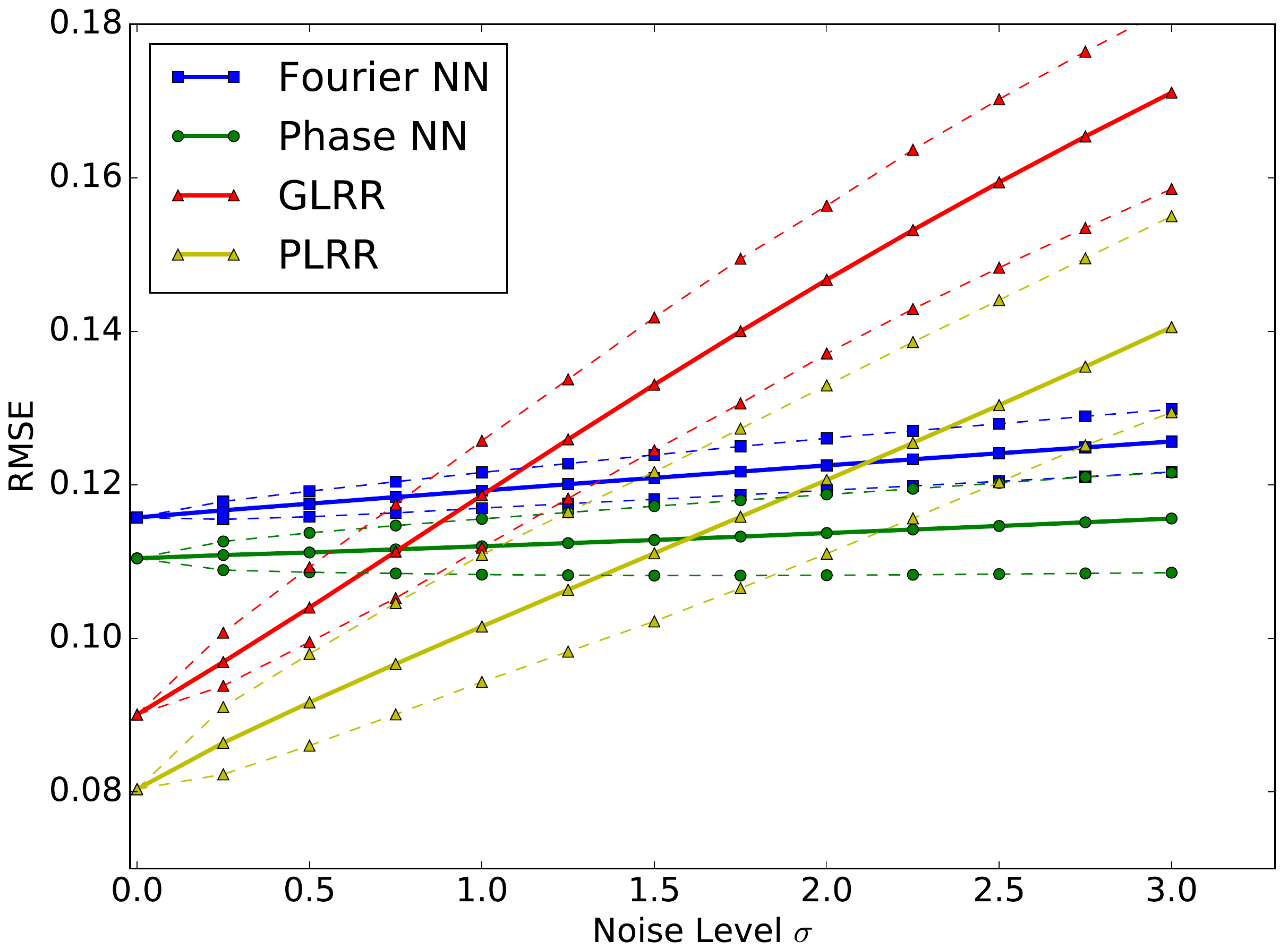}
  \vspace{0.3cm}
  \captionof{figure}{RMSE on the Aerosol test set, corrupted by various levels of noise averaged over $100$ runs, with the $5^{th}$ and the $95^{th}$ percentile. The noiseless case is shown with one run. RMSE from mean is $0.206$.}
  \label{fig:phasegkrrrff}
\end{minipage}
\end{figure}
\subsubsection{Aerosol Dataset}   To demonstrate the phase features invariance to SPD noise component, we use the Aerosol MISR1 dataset also studied by \cite{Szabo2015} and \cite{Wang2012} and consider a situation with \textit{covariate shift} \cite{Quinonero2009} on distribution inputs: the testing data is impaired by additive SPD components different to that in the training data. Here, we have an aerosol optical depth (AOD) multi-instance learning problem with $800$ bags, where each bag contains $100$ randomly selected multispectral (potentially cloudy) pixels within $20$km radius around an AOD sensor. The label $y_i$ for each bag is given by the AOD sensor measurements and each sample $x_i$ is $16$-dimensional. This can be understood as a distribution regression problem where each bag is treated as a set of samples from some distribution. 

We use $640$ bags for training and $160$ bags for testing. Here in the bags for testing \textit{only}, we add varying levels of Gaussian noise $\epsilon \sim \mathcal{N}(0, Z)$ to each bag, where $Z$ is a diagonal matrix with diagonal components $z_i \sim U[0, \sigma v_i]$ with $v_i$ being the empirical variance in dimension $i$ across all samples, accounting for different scales across dimensions. For comparisons, we consider linear ridge regression on embeddings with respect to a Gaussian kernel, approximated with RFF (GLRR) as described in section \ref{sec:background} (i.e. a linear kernel is applied on approximate embeddings), linear ridge regression on phase features (PLRR) (i.e. normalisation step is applied to obtain \eqref{eqn:phase_em}), and also the phase and Fourier neural networks (NN), described in Appendix \ref{sec:NNappendix}, tuning all hyperparameters with 3-fold cross validation. With the same model, we now measure Root Mean Square Error (RMSE) $100$ times with various noise-corrupted test sets and results are shown in figure \ref{fig:phasegkrrrff}. It is also noted that a second level non-linear kernel $\tilde{K}$ does not improve performance significantly on this problem \cite{Szabo2015}. 

We see that GLRR and PLRR are competitive (see Appendix Table \ref{tab:aerosol}) in the noiseless case, and these clearly outperform both the Fourier NN and Phase NN (likely due to the small size of the dataset). For increasing noise, the performance of GLRR degrades significantly, and while the performance of PLRR degrades also, the model is much more robust under additional SPD noise. In comparison, the Phase NN implementation is almost insensitive to covariate shift in the test sets, unlike the performance of PLRR, highlighting the importance of learning discriminative frequencies $w$ in a very low signal-to-noise setting.

It is noted that the Fourier NN performs similarly to that of the Phase NN on this example. Interestingly, discriminative frequencies learnt on the training data correspond to Fourier features that are nearly normalised (i.e. they are close to unit norm - see Figure \ref{fig:freq_aero} in the Appendix). This means that the Fourier NN has \emph{learned to be approximately invariant} based on training data, indicating that the original Aerosol data potentially has irrelevant SPD noise components. This is reinforced by the nature of the dataset (each bag contains $100$ randomly selected potentially cloudy pixels, known to be noisy \cite{Wang2012}) and no loss of performance from going from GLRR to PLRR. The results highlights that phase features are stable under additive SPD noise.%, even under such a difficult setting.

\subsubsection{Dark Matter Dataset}   
\begin{figure}[t]
\centering
\begin{minipage}{.475\textwidth}
  \centering
 \captionof{table}{Mean Square Error (MSE) on dark matter dataset for $500$ runs with $5^{th}$ and $95^{th}$ percentile.}
 \vspace{0.2cm}
\label{tab:dark}
\begin{normalsize}
\begin{tabular}{lccr}
    \toprule
  Algorithm  & MSE \\
           \midrule
  Mean & $0.16$ \\
    PLRR &  $\mathbf{0.021} \ (0.018, 0.024)$ \\
     GLRR & $0.033 \ (0.030, 0.037 ) $\\
     \midrule
       LGRR &   $0.032 \ (0.028, 0.036)$ \\
     PGRR & $0.021 \ (0.017, 0.024)$\\
      GGRR &$ \mathbf{0.018}  \ (0.015, 0.019)$  \\
    \bottomrule
\end{tabular}
\end{normalsize}
  \end{minipage}%
\begin{minipage}{.475\textwidth}
\centering
  \includegraphics[scale = 0.30]{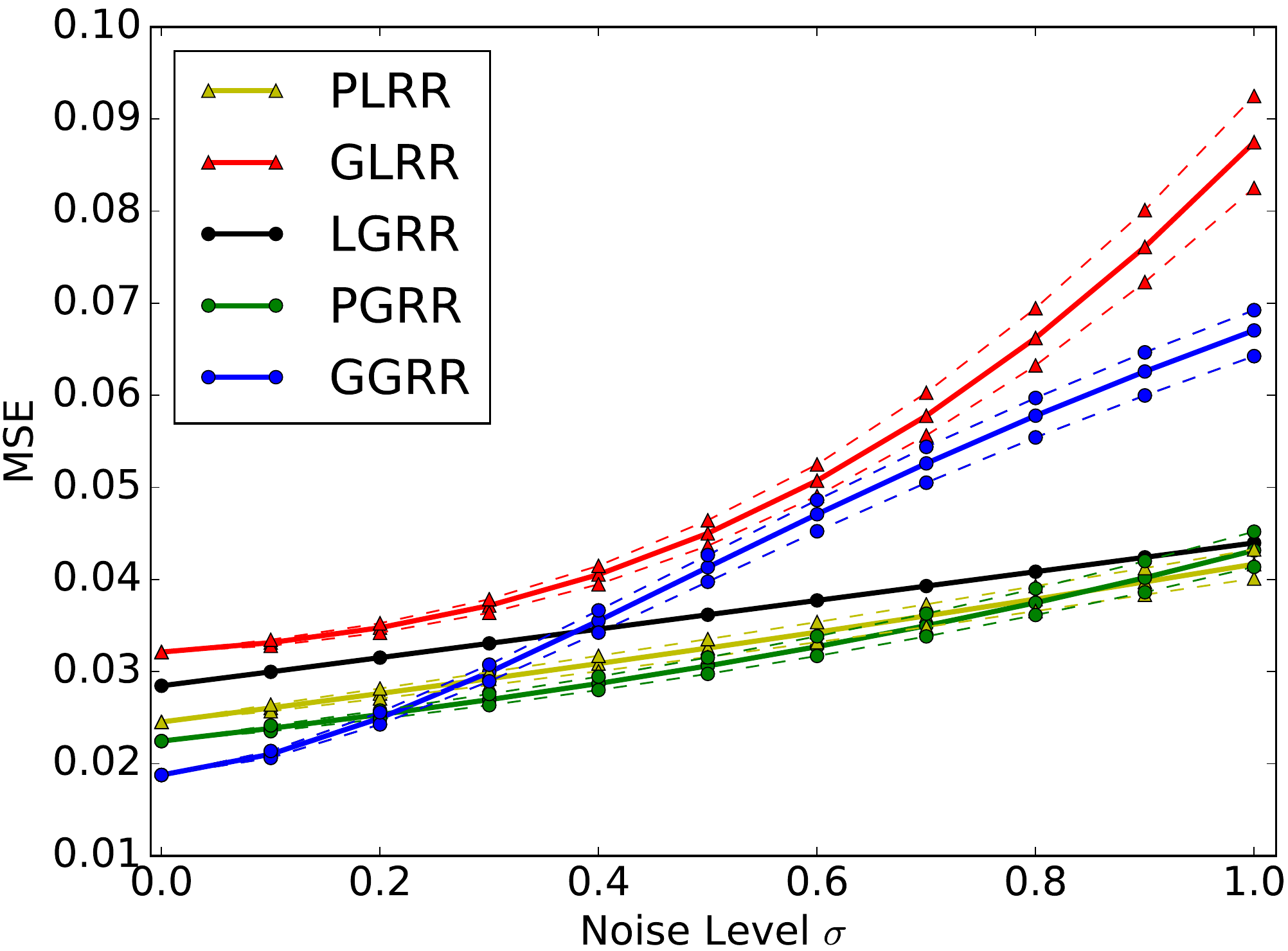}
  \vspace{0.2cm}
  \captionof{figure}{MSE with various levels of noise added on test set, with $5^{th}$ and $95^{th}$ percentile.}
  \label{fig:astro_shift}
\end{minipage}
\vspace{0cm}
\end{figure}
We now study the use of phase features on the dark matter dataset, composing of a catalog of galaxy clusters. In this setting, we would like to predict the total mass of galaxy clusters, using the dispersion of velocities in the direction along our line of sight. In particular, we will use the `ML1' dataset, as obtained from the authors of \cite{Ntampaka2015,Ntampaka2016}, who %has performed such a study and
constructed a catalog of massive halos from the MultiDark \texttt{mdpl} simulation \cite{multidark}. The dataset contains $5028$ bags, with each sample consisting of its sub-object velocity and its mass label in $\mathbb{R}$. By viewing each galaxy cluster at multiple lines of sights, we obtain $15\,000$ bags, using the same experimental setup as in \cite{law2017bdr}. For experiments, we use approximately $9000$ bags for training, and $3000$ bags each for validation and testing, keeping those of multiple lines of sight in the same set. As before, we use GLRR and PLRR and we also include in comparisons methods with a second level Gaussian kernel (with RFF) applied to phase features (PGRR) and to approximate embeddings (GGRR). For a baseline, we also include a first level linear kernel (equivalent to representing each bag with its mean), before applying a second level gaussian kernel (LGRR). We use the same set of randomly sampled frequencies across the methods, tuning for the scale of the frequencies and for regularisation parameters. 

Table \ref{tab:dark} shows the results of the methods across $10$ different data splits, with $50$ sets of randomised frequencies for each data split. We see that PLRR is significantly better than GLRR. This suggests that under this model structure, by removing SPD components from each bag, we can target the underlying signal and obtain superior performance, highlighting the applicability of phase features. Considering a second level gaussian kernel, we see that the GGRR has a slight advantage over PGRR, with PGRR performing similar to PLRR. This suggests that the SPD components of the distribution of sub-object velocity may be useful for predicting the mass of a galaxy cluster if an additional nonlinearity is applied to embeddings -- whereas the benefits of removing them outweigh the signal present in them without this additional nonlinearity. %This also suggests in practice to attempt both Fourier and phase features or a conjunction of them.
To show that indeed the phase features are robust to SPD components, we perform the same covariate shift experiment as in the aerosol dataset, with results given in Figure \ref{fig:astro_shift}. Note that LGRR is robust to noise, as each bag is represented by its mean.

\section{Conclusion}
\label{sec: conclu}
No dataset is immune from measurement noise and often this noise differs across different data generation and collection processes. When measuring distances between distributions, can we disentangle the differences in noise from the differences in the signal? We considered two different ways to encode invariances to additive symmetric noise in those distances, each with different strengths: a nonparametric measure of asymmetry in paired sample differences and a weighted distance between the empirical phase functions. The former was used to construct a hypothesis test on whether the difference between the two generating processes can be explained away by the difference in postulated noise, whereas the latter allowed us to introduce a flexible framework for invariant feature construction and learning algorithms on distribution inputs which are robust to measurement noise and target underlying signal distributions.
\section*{Acknowledgements}
We thank Dougal Sutherland for suggesting the use of of the dark matter dataset, Michelle Ntampaka for providing the catalog, as well as Ricardo Silva, Hyunjik Kim and Kaspar Martens for useful discussions. This work was supported by the EPSRC and MRC through the OxWaSP CDT programme (EP/L016710/1). C.Y. and H.C.L.L. also acknowledge the support of the MRC Grant No. MR/L001411/1. \\ \\ The CosmoSim database used in this paper is a service by the Leibniz-Institute for Astrophysics Potsdam (AIP). The MultiDark database was developed in cooperation with the Spanish MultiDark Consolider Project CSD2009-00064. The authors gratefully acknowledge the Gauss Centre for Supercomputing e.V. (www.gauss-centre.eu) and the Partnership for Advanced Supercomputing in Europe (PRACE, www.prace-ri.eu) for funding the MultiDark simulation project by providing computing time on the GCS Supercomputer SuperMUC at Leibniz Supercomputing Centre (LRZ, www.lrz.de).
\clearpage
\bibliographystyle{plain2}
\bibliography{phase_icml}
\renewcommand{\thefigure}{\Alph{section}.\arabic{figure}}
\renewcommand{\thetable}{\Alph{section}.\arabic{table}}
\renewcommand{\thealgorithm}{\Alph{section}.\arabic{algorithm}}
\newpage

\appendix
\section{Different Indecomposable Distributions Can Coincide in Phase}
\setcounter{figure}{0}    
\setcounter{table}{0}    
\setcounter{algorithm}{0}
\label{sec:indecom}
Let $X$ and $Y$ be (univariate) random variables with densities
$$f_X(x)=\frac{1}{\sqrt{2\pi}}x^2\exp(-x^2/2),\quad f_Y(x)=\frac{1}{2}|x|\exp(-|x|).$$
Then it can be directly checked that their characteristic functions are given by
$$\varphi_X(\omega)=(1-\omega^2)\exp(-\omega^2/2),\quad \varphi_Y(\omega)=\frac{1-\omega^2}{(1+\omega^2)^2}.$$
Thus, the phase functions coincide and are equal to
$$
\rho_X(\omega)=\rho_Y(\omega)=
\begin{cases}
+1,|\omega|<1,\\
-1,|\omega|>1,\\
\text{undefined},\omega\in\{-1,1\}.\\
\end{cases}
$$
However, it is can also checked that even though they are symmetric, $X$ and $Y$ are indecomposable, cf. e.g. \cite{linnik1977decomposition}, which use a related but distinct notion of indecomposability of random variables. The plots of the densities and characteristic functions of $X$ and $Y$ are given in Fig. \ref{fig:counter}.
\begin{figure}[!ht]
\centering
\includegraphics[width=0.4\textwidth]{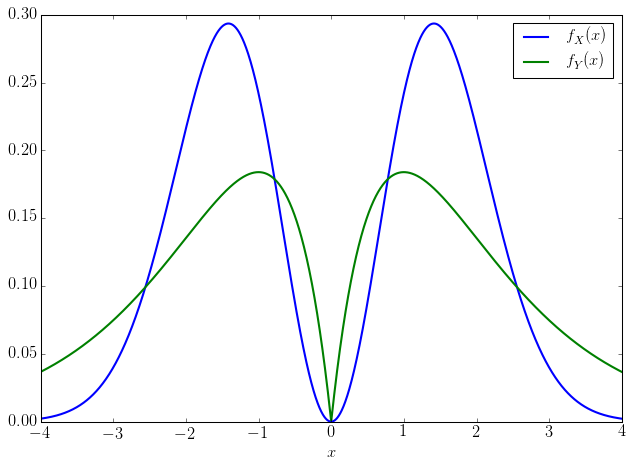}\includegraphics[width=0.4\textwidth]{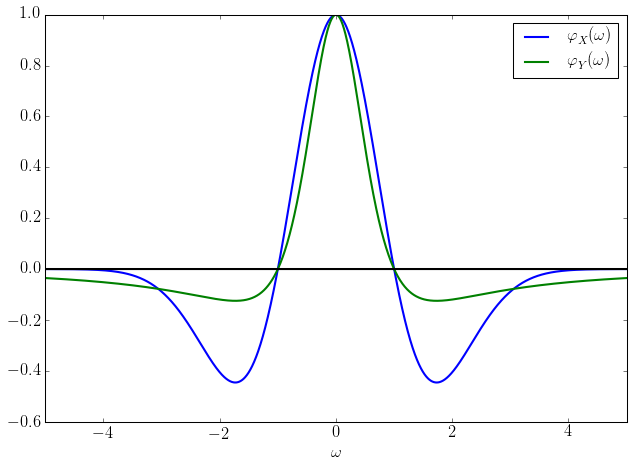}
\caption{Example of two indecomposable distributions which have the same phase function. {\bf Left}: densities. {\bf Right}: charactersitic functions.}
\label{fig:counter}
\end{figure}

\section{Phase Discrepancy and Asymmetry in Paired Differences Proofs}
\setcounter{figure}{0}    
\setcounter{table}{0}    
\setcounter{algorithm}{0}
\label{sec:proofs}
In this section, we will provide further details of the definitions, calculations and proofs in section \ref{sec:main} and \ref{sec:paired differences}. Phase discrepancy is defined as the weighted $L_{2}$-distances between the phase functions,
i.e.
\[
\text{PhD}(X,Y)=\int\left|\rho_{X}\left(\omega\right)-\rho_{Y}\left(\omega\right)\right|^{2}d\Lambda\left(\omega\right),
\]
for some positive measure $\Lambda$ (w.l.o.g. a probability measure).
Phase discrepancy measures how much $X$ and $Y$ differ up to an independent SPD noise component. We note that while the form of the PhD is motivated by that of MMD (weighted $L_{2}$-distances between the characteristic functions), relating it to the properties of the corresponding kernel and its RKHS is not straightforward. For example, constructing a PhD interpretation as a supremum over the RKHS unit ball (which is often how MMD is introduced) is immediate only for the case where indecomposable parts are point masses. Namely, if $X=x_0+U$ and $Y=y_0+V$, i.e. indecomposable parts are almost surely constant vectors $x_0$ and $y_0$, then
$$
PhD(X,Y)=\Vert k(\cdot,x_0)-k(\cdot,y_0) \Vert_{\mathcal H_k}^2 = \sup_{\Vert f \Vert_{\mathcal H_k}\leq 1}\left|f(x_0)-f(y_0)\right|^2,
$$
by virtue of $\rho_X(\omega)=e^{i\omega^\top x_0}=\varphi_{x_0}(\omega)$. In other cases, while it is clear that the spectral properties of the kernel still regulate the amount of frequency content that is used, one obtains the RKHS distance between the kernel convolutions of the inverse Fourier transforms of the phase functions so the interpretation is less clear.

Below, we provide the proofs of the propositions from the main text.
\begin{prop}
$$
\text{PhD}(X,Y)=2-2\int\frac{\mathbb{E}\cos\left(\omega^{\top}\left(X-Y\right)\right)}{\sqrt{\mathbb{E}\cos\left(\omega^{\top}\left(X-X'\right)\right)\mathbb{E}\cos\left(\omega^{\top}\left(Y-Y'\right)\right)}}d\Lambda(\omega).
$$
\end{prop}
\begin{proof}
\begin{eqnarray*}
\text{PhD}(X,Y) & = & \int\left|\rho_{X}\left(\omega\right)-\rho_{Y}\left(\omega\right)\right|^{2}d\Lambda\left(\omega\right)\\
 & = & \int\left|\rho_{X}\left(\omega\right)\right|^{2}d\Lambda\left(\omega\right)+\int\left|\rho_{Y}\left(\omega\right)\right|^{2}d\Lambda\left(\omega\right) -\int\left(\rho_{X}\overline{\rho_{Y}}+\overline{\rho_{X}}\rho_{Y}\right)d\Lambda\\
 & = & 2-\int\frac{\varphi_{X}\overline{\varphi_{Y}}+\overline{\varphi_{X}}\varphi_{Y}}{\left|\varphi_{X}\right|\left|\varphi_{Y}\right|}d\Lambda\\
 & = & 2-2\int\frac{\varphi_{Z}}{\sqrt{\varphi_{X-X'}\varphi_{Y-Y'}}}d\Lambda,
\end{eqnarray*}
where $X$ and $X'$ are iid, $Y$ and $Y'$ are iid and $Z$ is an
equal mixture of $X-Y$ and $Y-X$. Indeed, 
\[
\varphi_{X}\overline{\varphi_{Y}}+\overline{\varphi_{X}}\varphi_{Y}=\varphi_{X-Y}+\varphi_{Y-X}=2\varphi_{Z},
\]
and
\[
\varphi_{X-X'}=\varphi_{X}\overline{\varphi_{X}}=\left|\varphi_{X}\right|^{2}.
\]
Note that $X-X'$,$Y-Y'$ and $Z$ are all symmetric. Thus, 
\begin{eqnarray*}
\varphi_{Z}(\omega) &= &\mathbb{E}\left[\cos\left(\omega^{\top}Z\right)\right]=\frac{1}{2}\mathbb{E}\left[\cos\left(\omega^{\top}\left(X-Y\right)\right)\right]+\frac{1}{2}\mathbb{E}\left[\cos\left(\omega^{\top}\left(Y-X\right)\right)\right] \\ &= &\mathbb{E}\left[\cos\left(\omega^{\top}\left(X-Y\right)\right)\right].
\end{eqnarray*}
Substituting provides us the result.
\end{proof}
\begin{prop}
\label{sec: pos_kernel}
$K_{\omega}\left(\mathsf{P}_{X},\mathsf{P}_{Y}\right)=\left(\frac{\mathbb{E}\xi_{\omega}(X)}{\left\Vert \mathbb{E}\xi_{\omega}(X)\right\Vert }\right)^{\top}\left(\frac{\mathbb{E}\xi_{\omega}(Y)}{\left\Vert \mathbb{E}\xi_{\omega}(Y)\right\Vert }\right)$
is a positive definite kernel on probability measures $\forall \omega$, where here $\xi_{\omega}\left(x\right)=\left[\cos\left(\omega^{\top}x\right),\sin\left(\omega^{\top}x\right)\right]$, and so is $K\left(\mathsf{P}_{X},\mathsf{P}_{Y}\right)=\int K_\omega\left(\mathsf{P}_{X},\mathsf{P}_{Y}\right)d\Lambda(\omega)$ for any positive measure $\Lambda$.
\end{prop}

\begin{proof}
Define a feature map $\xi_{\omega}:\mathbb{\mathcal{X}\to\mathbb{R}}^{2}$
with $\xi_{\omega}\left(x\right)=\left[\cos\left(\omega^{\top}x\right),\sin\left(\omega^{\top}x\right)\right]$,
which induces a kernel on $\mathcal{X}$ given by $k_{\omega}(x,y)=\cos\left(\omega^{\top}\left(x-y\right)\right)$.
Then $\kappa_{\omega}\left(\mathsf{P}_{X},\mathsf{P}_{Y}\right)=\mathbb{E}\cos\left(\omega^{\top}\left(X-Y\right)\right)=\mathbb{E}k_{\omega}(X,Y)=\left(\mathbb{E}\xi_{\omega}(X)\right)^{\top}\mathbb{E}\xi_{\omega}(Y)$
is a valid kernel on probability measures and so is the normalised
kernel 
\[
K_{\omega}\left(\mathsf{P}_{X},\mathsf{P}_{Y}\right)=\frac{\kappa_{\omega}\left(\mathsf{P}_{X},\mathsf{P}_{Y}\right)}{\sqrt{\kappa_{\omega}\left(\mathsf{P}_{X},\mathsf{P}_{X}\right)\kappa_{\omega}\left(\mathsf{P}_{Y},\mathsf{P}_{Y}\right)}}=\left(\frac{\mathbb{E}\xi_{\omega}(X)}{\left\Vert \mathbb{E}\xi_{\omega}(X)\right\Vert }\right)^{\top}\left(\frac{\mathbb{E}\xi_{\omega}(Y)}{\left\Vert \mathbb{E}\xi_{\omega}(Y)\right\Vert }\right),
\]
where we used that $\mathbb{E}\cos\left(\omega^{\top}\left(X-X'\right)\right)=\left(\mathbb{E}\xi_{\omega}(X)\right)^{\top}\mathbb{E}\xi_{\omega}(X')=\left\Vert \mathbb{E}\xi_{\omega}(X)\right\Vert ^{2}$. For the last claim, simply note that integrating through the positive measure preserves positive semidefinitess, i.e. $\sum\alpha_{i}\alpha_{j}K(\mathsf{P}_{i},\mathsf{P}_{j})=\int\left(\sum\alpha_{i}\alpha_{j}K_{\omega}(\mathsf{P}_{i},\mathsf{P}_{j})\right)d\Lambda\left(\omega\right)\geq0$.
\end{proof}
As a direct corollary,
\begin{prop}
\textup{$\text{PhD}(X,Y)=2-2K\left(\mathsf{P}_{X},\mathsf{P}_{Y}\right)=2\int\left(1-\left(\frac{\mathbb{E}\xi_{\omega}(X)}{\left\Vert \mathbb{E}\xi_{\omega}(X)\right\Vert }\right)^{\top}\left(\frac{\mathbb{E}\xi_{\omega}(Y)}{\left\Vert \mathbb{E}\xi_{\omega}(Y)\right\Vert }\right)\right)d\Lambda(\omega).$}
\end{prop}

\begin{prop} 
Under the null hypothesis, $X - Y \text{ is SPD } \iff X_0 \eqd Y_0.$
\end{prop}
\begin{proof}
Under $H_0$, since $X_0$ has the same distribution as $Y_0$, then so do $X-Y=X_{0}-Y_{0}+U-V$ and $Y-X=Y_{0}-X_{0}+V-U$ as $U-V$ is symmetric. Moreover, $\varphi_{X-Y}=|\varphi_{X_0}|^2\varphi_U\varphi_V>0$, so $X-Y$ is SPD. Conversely, if we assume that $X-Y$ is SPD, i.e. $\varphi_X \overline{\varphi_Y} > 0$, then $\rho_{X_0} \overline{\rho_{Y_0}} > 0$. Since $ | \rho_{X_0}|= |\rho_{Y_0} |= 1$, this implies that $\rho_{X_0} = \rho_{Y_0}$, and hence $X_0\eqd Y_0$, since we assumed that $X_0$ and $Y_0$ belong to $\mathcal{P}(\mathbb{R}^d)$. Hence, we have that $X - Y \text{ is SPD } \iff X_0 \eqd Y_0.$ 
\end{proof}

\section{Paired Differences}
\setcounter{figure}{0}    
\setcounter{table}{0}    
\setcounter{algorithm}{0}
\label{sec:paried_append}
Another way to measure asymmetry of the difference between random vectors $X$ and $Y$ is to use $\text{MMD}(X-Y,Y-X)$ instead
of $\text{PhD}(X,Y)$. However, this quantity is not invariant, i.e., $\text{MMD}(X-Y,Y-X)\neq\text{MMD}(X_{0}-Y_{0},Y_{0}-X_{0})$,
and in fact the values will heavily depend on the distributions of
$U$ and $V$. We note that 
\begin{eqnarray*}
\varphi_{X-Y}(\omega)-\varphi_{Y-X}(\omega) & = & 2i\mathbb{E}\sin\left(\omega^{\top}\left(X-Y\right)\right),
\end{eqnarray*}
so that we are effectively measuring the size of the imaginary part
of the characteristic function of $X-Y$ (which should not be there
if it is symmetric). There are several different ways in which we can write this quantity:

\begin{eqnarray*}
\text{MMD}(X-Y,Y-X) & = & \left\Vert \mathbb{E}k(\cdot,X-Y)-\mathbb{E}k(\cdot,Y-X)\right\Vert _{\mathcal{H}_{k}}^{2}\\
 & = & \int\left|\varphi_{X}\overline{\varphi_{Y}}-\overline{\varphi_{X}}\varphi_{Y}\right|^{2}d\Lambda\\
 & = & 4\int\left[\mathbb{E}\sin\left(\omega^{\top}\left(X-Y\right)\right)\right]^{2}d\Lambda(\omega)\\
 & = & \int\left|\varphi_{X}\right|^{2}\left|\varphi_{Y}\right|^{2}\left(2-\frac{\varphi_{X}\overline{\varphi_{Y}}}{\overline{\varphi_{X}}\varphi_{Y}}-\frac{\overline{\varphi_{X}}\varphi_{Y}}{\varphi_{X}\overline{\varphi_{Y}}}\right)d\Lambda.
\end{eqnarray*}
The last expression indicates that this quantity is affected by the amplitude of the individual
characteristic functions, experimental details to show this empirically can be found in \ref{sec:pdexp}. Moreover, the quantity does not appear to lend itself to the \emph{feature on distributions} formalism, i.e. we were unable to derive some Hilbert space features $\Upsilon(\mathsf P)\in \mathcal H$ such that $\text{MMD}(X-Y,Y-X)=\Vert \Upsilon(\mathsf P_{X})-\Upsilon(\mathsf P_{Y})\Vert_{\mathcal H}^2$, and it is thus unclear whether this approach can be used to define a valid kernel on distributions. 
\section{Learning Discriminative Features}
\setcounter{figure}{0}    
\setcounter{table}{0}    
\setcounter{algorithm}{0}    
\label{sec:NNappendix}
\begin{figure}[H]
\centering
\begin{minipage}{.57\textwidth}
\centering
\begin{algorithm}[H]
   \caption{Phase/Fourier Neural Network}
   \label{alg:nn}
\begin{algorithmic}
   \STATE {\bfseries Input:} Batch of bag of samples $X \in \mathbb{R}^{b \times N \times p}$, where $b$ is the batch size, $N$ is the bag size and $p$ is the dimension
      \STATE{ \bfseries Output:} Classification or Regression Output
      \STATE {\bfseries 1.} Compute $f(X) = X W$ where $W \in \mathbb{R}^{p \times m}$
      \STATE{\bfseries 2.} Apply a $\sin$ and $\cos$ activation function $$l_1(X) = [ \sin(f(X)) \cos(f(X)) ]$$
       \STATE  {\bfseries 3.} Apply mean pooling operation over $N$, effectively computing $\hat{\mathbb{E}}\xi_{\omega_{i}}(X)$ for each $\omega_{i} \in \mathbb{R}^p$  $$\scalebox{1.0}{$ l_2(X)= \left[\hat{\mathbb{E}}\xi_{\omega_{1}}(X),\ldots,\mathbb{\hat{E}}\xi_{\omega_{m}}(X)\right] \in \mathbb{R}^{2m}$}$$
       
\STATE { \bfseries 4.} For Phase Neural Network, compute $\left\Vert \mathbb{\hat{E}}\xi_{\omega_{1}}(X)\right\Vert$ for each frequency and normalise to obtain:
$$ \scalebox{1.0}{ $l_3(X)= \left[\frac{\hat{\mathbb{E}}\xi_{\omega_{1}}(X)}{\left\Vert \mathbb{\hat{E}}\xi_{\omega_{1}}(X)\right\Vert },\ldots,\frac{\mathbb{\hat{E}}\xi_{\omega_{m}}(X)}{\left\Vert \mathbb{\hat{E}}\xi_{\omega_{m}}(X)\right\Vert }\right]$}$$ 
\STATE{\bfseries 5.} Batch Normalisation Layer
\STATE{\bfseries 6.} Output layer
     \end{algorithmic}
\end{algorithm}
  \end{minipage}%
  \hspace{0.2cm}
\begin{minipage}{.40\textwidth}
\centering
  \includegraphics[scale = 0.164]{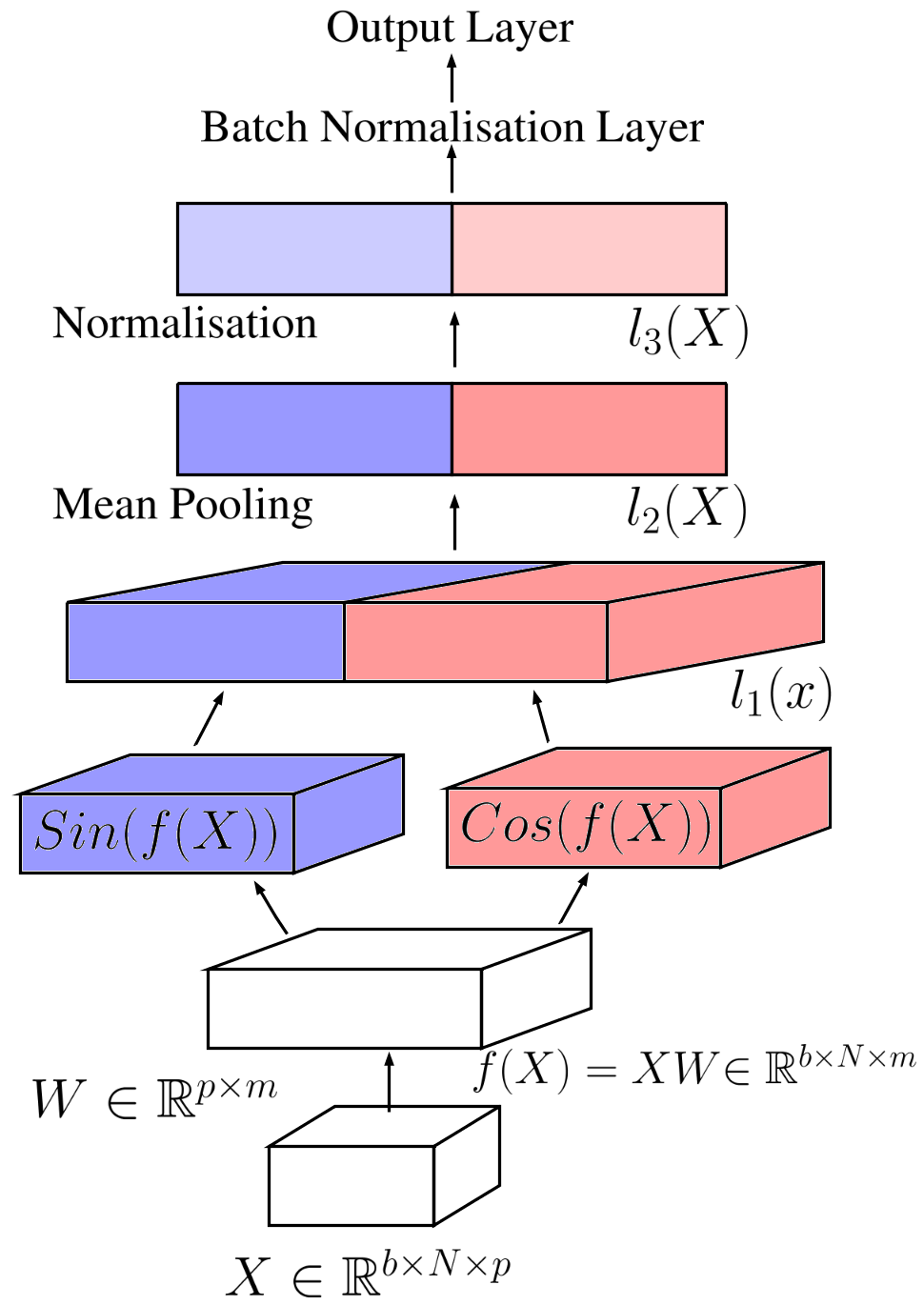}
  \captionof{figure}{Main structure of the phase neural network}
  \label{fig: pnn}
\end{minipage}
\end{figure}
Algorithm \ref{alg:nn} shows the phase Neural Network (phase NN) and the Fourier Neural Network (Fourier NN), where the latter can be obtained by simply removing step $4$ in the algorithm. Although the batch normalisation is not required, it is highly recommended for faster training of the network \cite{ioffe2015batch}, due to the normalisation for the phase neural network in step $5$ of the algorithm. 
Because of the neural network structure, we can take advantage of the rich literature, as well as alter the network in order to target a variety of different problems. For example, setting now the loss function as the squared loss, cross entropy or pinball loss, we can solve tasks in regression, classification or quantile regression on distributional inputs with discriminative frequencies. The Fourier neural network can also be extended to inputs in $\mathbb{R}^p$ for normal regression and classification problems by removing the mean pooling operation in step $3$ of the algorithm.
\section{Distribution Regression with Invariance for ABC}
\setcounter{figure}{0}    
\setcounter{table}{0}    
\setcounter{algorithm}{0}

\begin{algorithm}[H]
   \caption{Phase Regression, Fourier Regression}
   \label{alg:reg}
\begin{algorithmic}
   \STATE {\bfseries Input:} prior $\pi$ for $\theta$, data-generating process $P$, phase or Fourier features
      \STATE {\bfseries Output:} Phase or Fourier Regression Neural Network 
   \FOR{$i = 1, \dots, n$}
   \STATE Sample $\theta_i \sim \pi$
   \STATE Sample dataset $B_i =\{x_{ij}\}_{j =1}^{N}$ from $P( \cdot | \theta_i)$
      \ENDFOR
     \STATE Train Phase or Fourier neural network with $\{B_i, y_i\}^n_{i=1}$
\end{algorithmic}
\end{algorithm}

\begin{algorithm}[H]
   \caption{ Phase-ABC or Fourier-ABC}
   \label{alg:abc}
\begin{algorithmic}
   \STATE {\bfseries Input:} prior $\pi$ for $\theta$, data-generating process $P$, observed data $ B^\ast = \{x^\ast _j\}_{j=1}^{N^\ast}, \epsilon$, number of particles $K$
      \STATE{ \bfseries Output:} Weighted Posterior sample $\sum_k w_k \delta_{\theta_k}$
      \STATE {\bfseries 1.} Perform Phase or Fourier Regression, obtain $m(\cdot)$
      \STATE{\bfseries 2.} ABC
   \FOR{$k = 1, \dots, K$}
   \STATE Sample $\theta_k \sim \pi$
   \STATE Sample dataset $B_k =\{x_{kj}\}_j$ from $P( \cdot | \theta_k)$
   \STATE Compute $\widetilde{w}_k = \exp \left( - \dfrac{|| m(B_k) - m(B^\ast) ||_2^2}{\epsilon} \right)  $
      \ENDFOR
     \STATE $w_k = {\widetilde{w}_k}/{\sum_k \widetilde{w}_k} $
     \end{algorithmic}
\end{algorithm}
%\section{Applications}
\label{abc}
We have designed an explicit feature map for a bag of samples that can be used for any distribution regression problem. We now present its potential application to Approximate Bayesian Computation (ABC). Motivated by the approach of \cite{fearnhead2012constructing} and \cite{mitrovic2016dr}, we propose to use the phase features to construct an optimal summary statistic (under some loss function) for ABC. ABC is a Bayesian framework that allows us to approximate the posterior distribution of some parameter $\theta$ by approximating the likelihood function through simulations. To capture this approximation of the likelihood function, simulated datasets from the model are compared with the observed data using some lower dimensional summary statistics. If the summary statistic is sufficient, then there is no loss of information when projecting the data onto lower dimensional space. In practice however, sufficient statistics are not available for complex models of interest and instead using the strategy of \cite{fearnhead2012constructing}, one can construct summary statistics that provide inference of $\theta$ which is optimal with respect to a given loss function. \\
\\
In particular, we will focus on the squared loss function as given by $L(\theta, \theta') = (\theta - \theta')^2$. \cite{fearnhead2012constructing} showed that under this loss, the posterior mean of the $\theta$ given observations $\bf X$ is in fact the optimal summary statistic of $\bf X$ for the ABC procedure. However, since this quantity can not be analytically computed, one approach is to estimate it by fitting a regression model from simulated data, some examples of this include the semi-automatic ABC \cite{fearnhead2012constructing} and DR-ABC \cite{mitrovic2016dr}. Here we focus on ideas from DR-ABC, which uses a kernel distribution regression approach, treating each simulated dataset (given $\theta$ simulated from the prior) as a bag of samples and taking its label to be $\theta$. After training the regression model, it proceeds to using it as a summary statistic as given in algorithm \ref{alg:abc}. The DR-ABC paper further proposed the conditional DR-ABC (CDR-ABC), which makes the assumption that only certain aspects of the data have an influence on $\theta$. By conditioning on such nuisance variables and then using conditional distribution regression (by embedding conditional distributions \cite{song2013kernel}), it can better account for the functional relationship inside the model. However, one problem with this approach is that the nuisance variables have to be observed directly, even for the true dataset, which may often not be the case. For example, consider the hierarchical model we used to illustrate the utility of phase features for regression below.
\begin{eqnarray}
\label{hiera_app}
\theta \sim \Gamma(\alpha, \beta), \quad Z &\sim&  U[0,\sigma], \quad \epsilon  \sim  \mathcal{N}(0, Z),\nonumber\\
\hspace*{-0.0cm}  X &\sim& \frac{\Gamma\big(\theta/2 , 1/2 \big)}{ 2 \theta}+ \epsilon, %{ \sqrt{2 \theta} }+ \epsilon, 
\end{eqnarray}
for some fixed values of $\alpha, \beta$ and $\sigma$. Here $\theta$ is the parameter we are interested in, $\epsilon$ is a latent noise variable (unobserved) and $X$ is the observation. Since neither $\epsilon$ nor $Z$ are observed on the true dataset, we can only use DR-ABC, not CDR-ABC. But DR-ABC then does not take into account the model structure which tells us that $\epsilon$ is irrelevant for inferring $\theta$, and it is thus likely to give poor performance for large values of $\sigma$. Hence, we propose to use phase features inside such regression model, which will be invariant to the noise variable $\epsilon$ which is an SPD component in observations. By using phase features for distribution regression, we should be able to better capture the functional relationship between $\theta$ and its corresponding dataset, a bag from $X|\theta$ and hence build better summary statistics for ABC. In some sense, this approach can be thought of as implicitly conditioning out the latent nuisance variable $\epsilon$, similarly as CDR-ABC does when it is observed. Furthermore, although we have chosen this example as an illustration, the phase features could be applied to many complex models with nuisance latent variables, even when we cannot write their contribution explicitly as here. The algorithms \ref{alg:reg} and \ref{alg:abc} shows the approach as in DR-ABC, but now replaced by our phase or Fourier regression approaches to compute summary statistics, and we denote these as Phase-ABC and Fourier-ABC. %Some experimental results can be found in \ref{sec:toy}.
 \section{Additional Results}
 \setcounter{figure}{0}    
\setcounter{table}{0}    
\setcounter{algorithm}{0}
\subsection{Asymmetry in Paired Differences Experiment}

\label{sec:pdexp}
\begin{figure}[ht!]
\centering
\includegraphics[scale = 0.25]{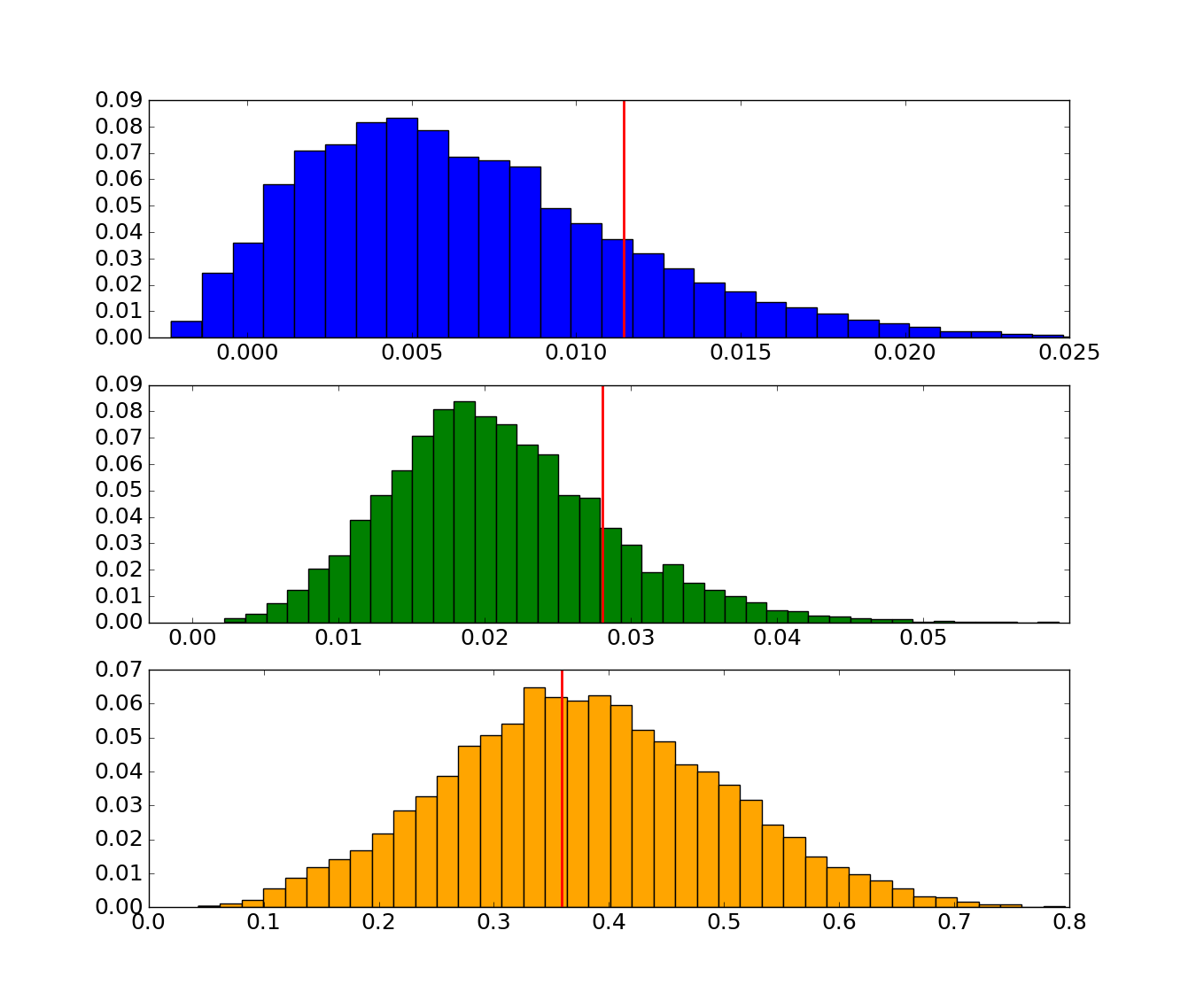}
\caption{ Histograms on various estimates for all pairs of bags with varying additive noise, red line denotes the noiseless case. \textbf{Top:} Estimated MMD on paired differences for all pair of bags, the red line given by the mean of the estimated MMD on paired differences for bags without noise. \textbf{Middle:} the squared distance between Fourier features (an estimate of MMD). \textbf{Bottom:} the squared distance between phase features (an estimate of PhD).}
\label{fig: mmdpd}
\end{figure}
\vspace{0.4cm}
While it performed well when testing the null hypothesis, the MMD on paired differences is not invariant to the additive SPD noise components under the alternative hypothesis. Using the synthetic experimental setup as before, we simulate $100$ noiseless bags from the two scaled $\chi^2$-distributions  $X_0\sim \chi^{2}(4)/4$ and $Y_0\sim \chi^{2}(8)/8$, where each bag contains $1000$ samples. We add varying levels of Gaussian noise to each bag, i.e. the bags are of the form $X_i = X_{0} +  \mathcal{N}(0,Z_i)$ and $Y_i = Y_{0} +  \mathcal{N}(0,W_i)$, where $Z_i,W_i \sim U[0,0.1]$. We compute the estimate of the MMD on paired differences, the squared distance between Fourier features (an estimate of MMD) and the squared distance between phase features (an estimate of PhD) for all pairs of bags. In all computations, we used the same set of frequencies $\{ w_i \}^{100}_{i=1}$ (sampled from a Gaussian distribution). We do the same for the noiseless samples (or use analytic expressions where available). The results are shown in figure \ref{fig: mmdpd}. We see that the MMD on paired differences is not invariant to SPD noise components (clearly, the noiseless case indicated by the red line has a much higher level of asymmetry than the noisy case where due to the presence of high levels of symmetric noise, differences often do appear symmetric). This is unlike the phase features, which maintain some level of invariance, the estimates stay away from 0 -- preserving the signal about the difference of indecomposable $\chi^2$ components -- and the mode is nearer the true value, even though there is clearly some variance, however this is expected as its PhD population expression is invariant, but not its estimator, furthermore the frequencies are sampled (with the median heuristic bandwidth) and not learnt. This suggests that phase features are more suitable for invariant learning on distributions than MMD on paired differences. The Fourier features are also given for comparison, but these are not expected to be invariant, as shown.
 \subsection{Characteristic and Phase Function Plots}
 \label{sec:char_phase}
\FloatBarrier
\begin{figure}[h]
\centering
\includegraphics[width=\textwidth]{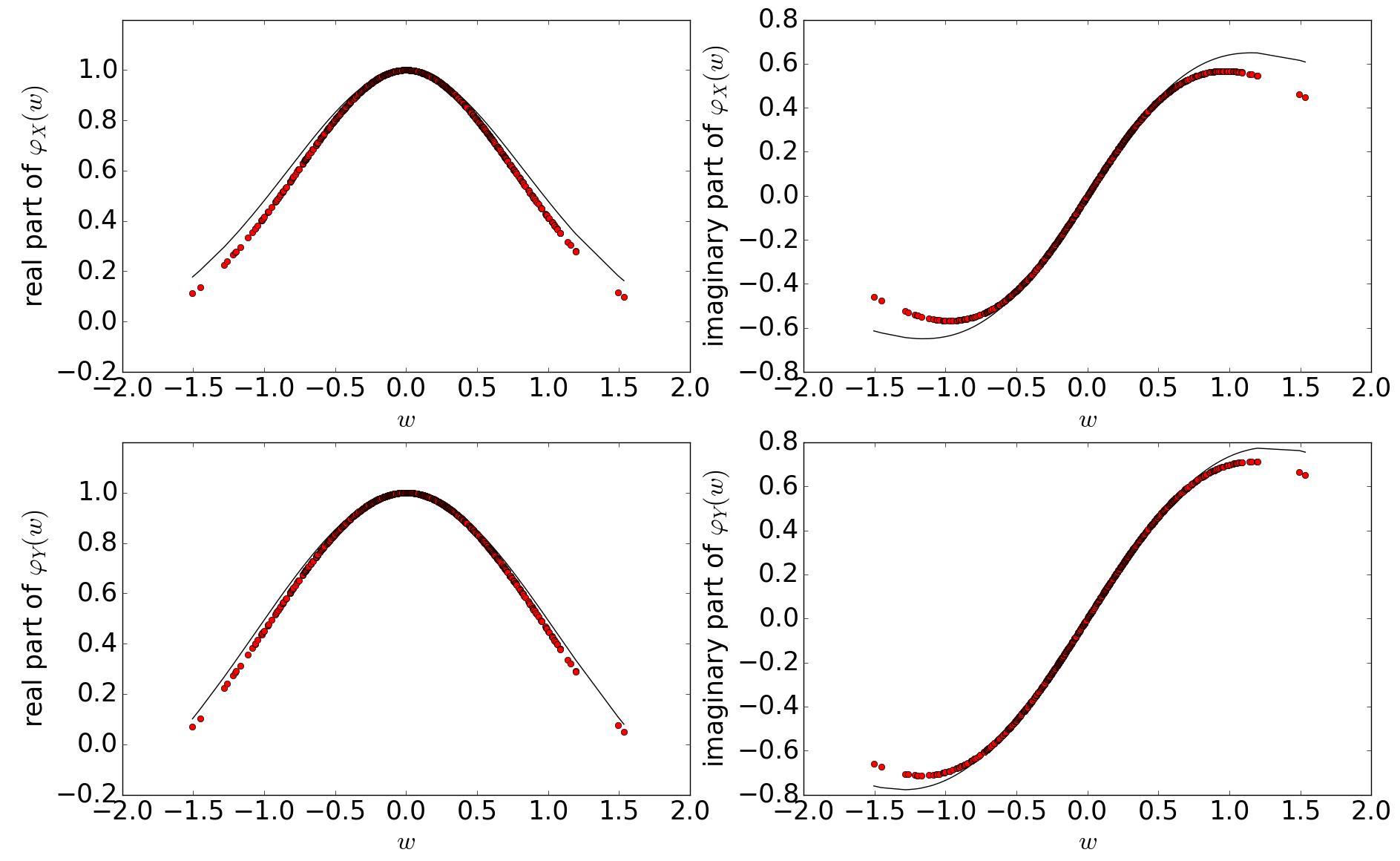}
\caption{The black line here correspond to the real and imaginary part of the true characteristic function of the $\chi^{2}(4)/4$ and $\chi^{2}(8)/8$ distribution, denoted $X, Y$ on the top and bottom graphs respectively. The red points denote the empirical characteristic function constructed with $750$ frequencies sampled from a Gaussian kernel with $\sigma = 2$ using a bag size of $1000$ observations, with some additional Gaussian noise.}
\label{fig: char_ratio}
\end{figure}
\begin{figure}[h]
\centering
\includegraphics[width=1\textwidth]{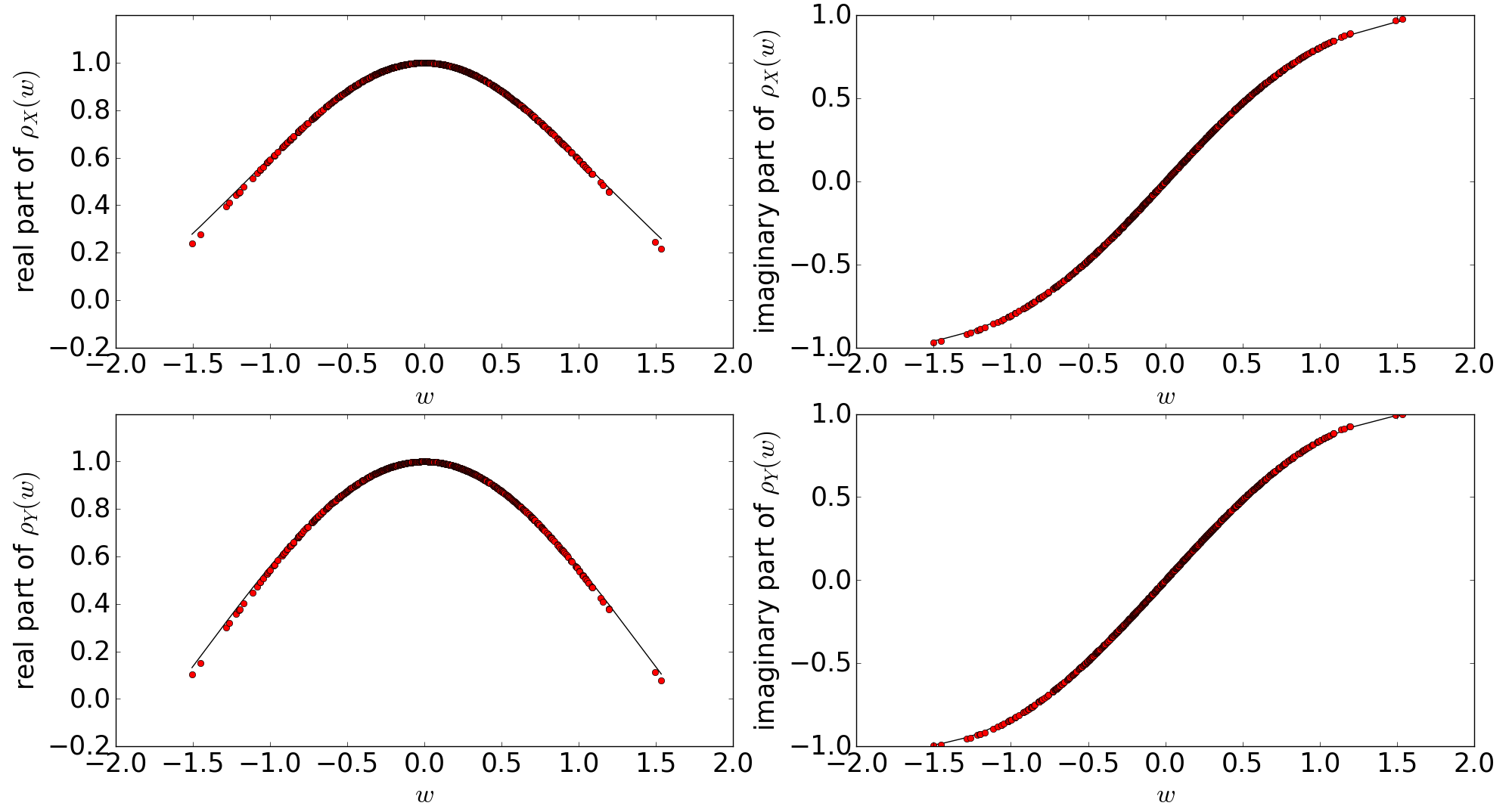}
\caption{The black line here correspond to the real and imaginary part of the true phase function of the $\chi^{2}(4)/4$ and $\chi^{2}(8)/8$ distribution, denoted $X, Y$ on the top and bottom graphs respectively. The red points denote the empirical phase function constructed with $750$ frequencies from a Gaussian kernel with $\sigma = 2$ using a bag size of $1000$ observations, with some additional Gaussian noise.} 
\label{fig: phase_ratio}
\end{figure}
\FloatBarrier
\begin{figure}[h]
\centering
\includegraphics[width=1\textwidth]{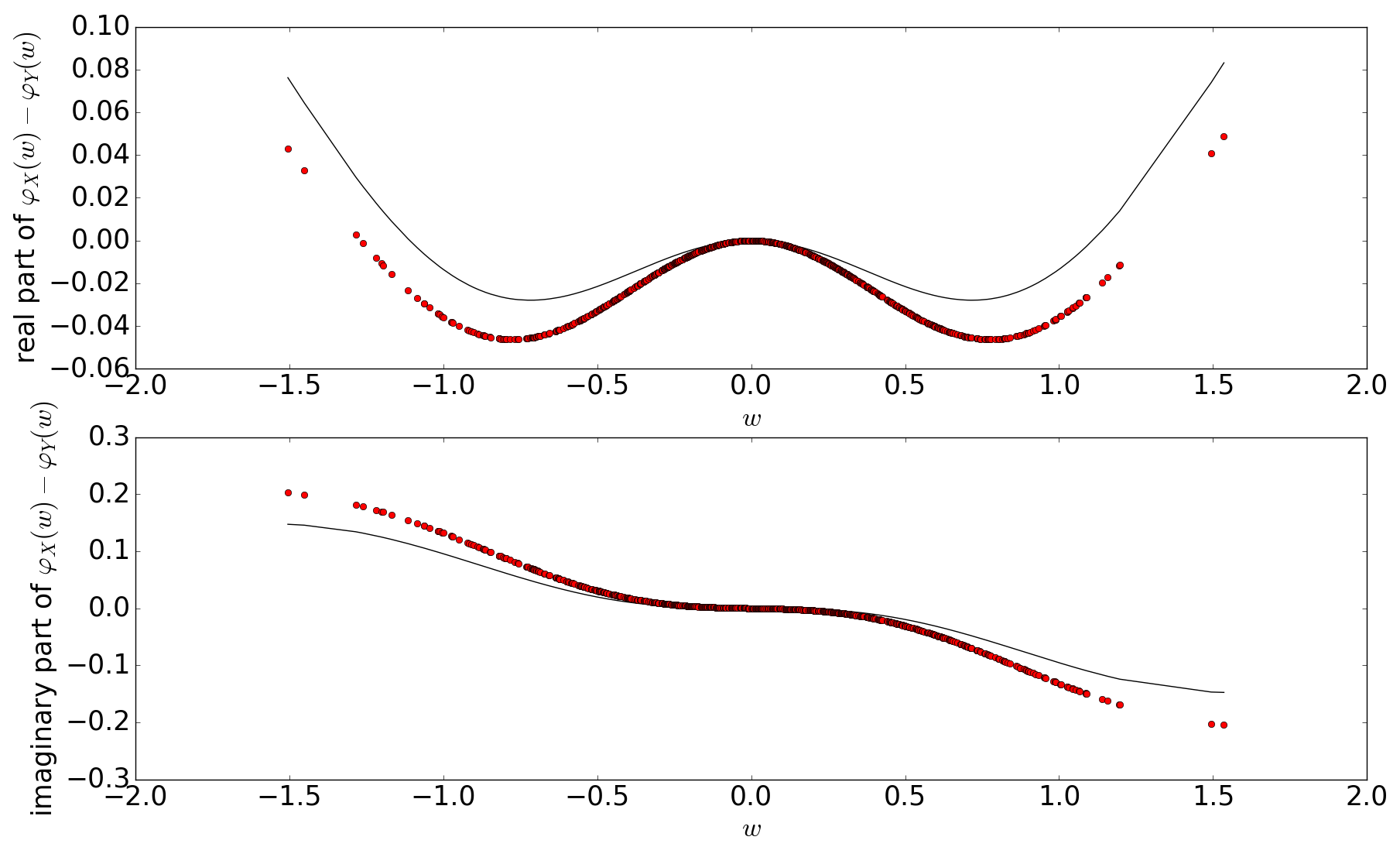}
\caption{The top and bottom graph denotes the difference in the real and imaginary part of the characteristic function for the $\chi^{2}(4)/4$ and $\chi^{2}(8)/8$ as in figure \ref{fig: char_ratio}. }
\end{figure}
\begin{figure}[h]
\centering
\includegraphics[width=\textwidth]{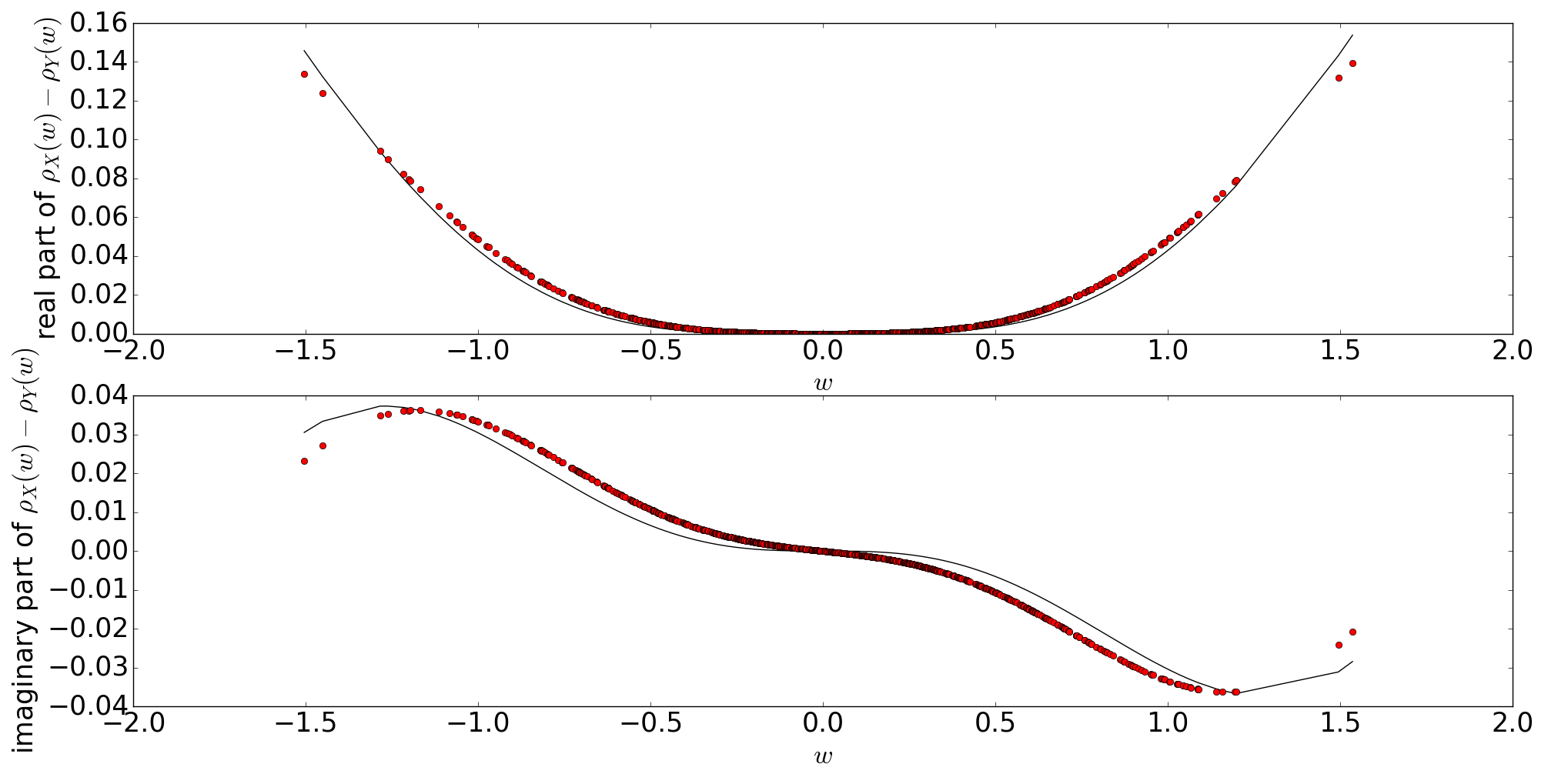}
\caption{The top and bottom graph denotes the difference in the real and imaginary part of the phase function for the $\chi^{2}(4)/4$ and $\chi^{2}(8)/8$ as in figure \ref{fig: phase_ratio}. }
\end{figure}
\FloatBarrier
 \subsection{Two-Sample Tests with Invariances}
\FloatBarrier
\subsubsection{Synthetic $\chi^2$ Dataset}
\begin{figure}[!ht]
    \centering
    \begin{subfigure}
        \centering
        \includegraphics[width=0.49\textwidth]{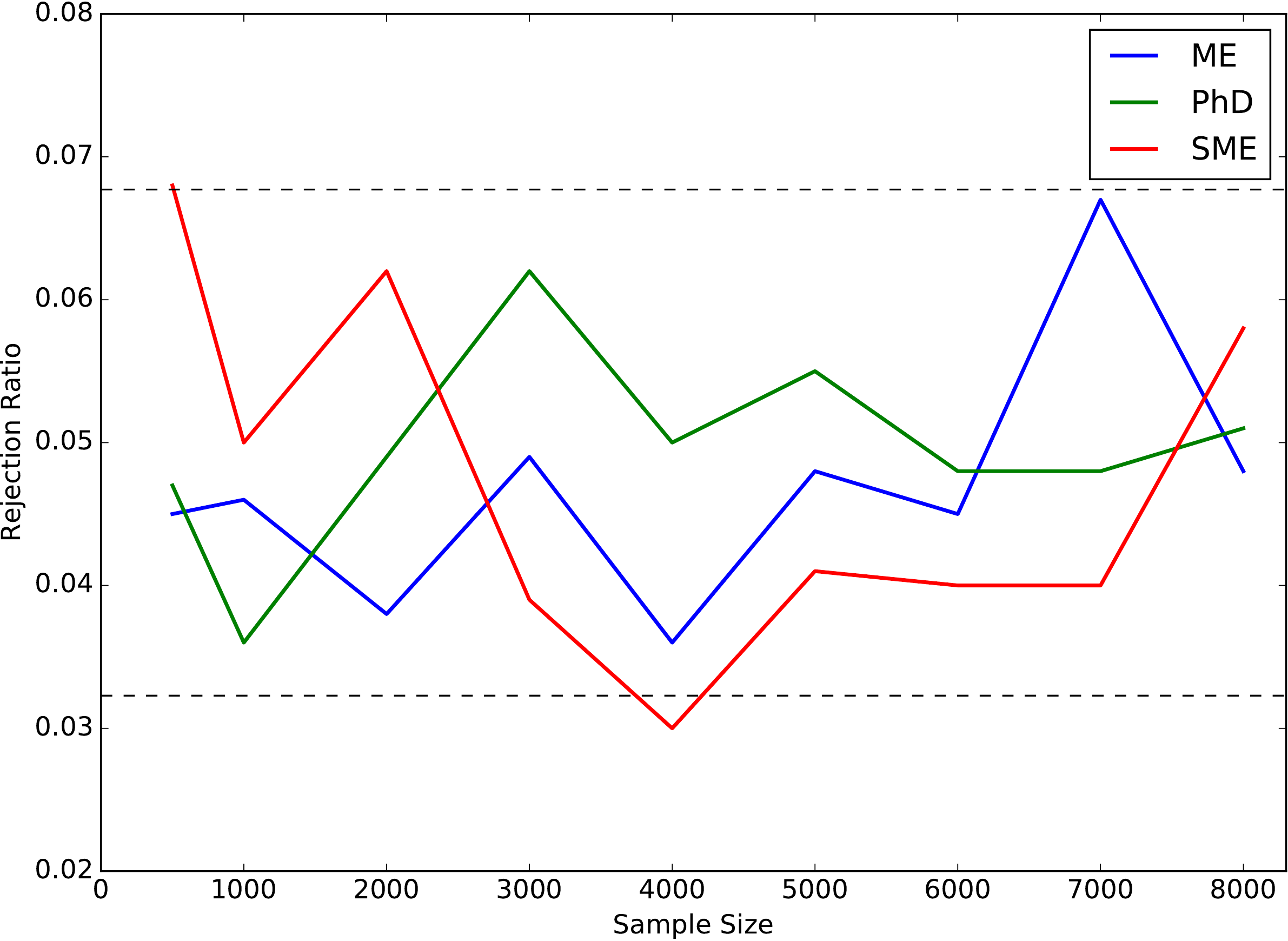}
    \end{subfigure}%
    ~ 
    \begin{subfigure}
        \centering
     \includegraphics[width=0.49\textwidth]{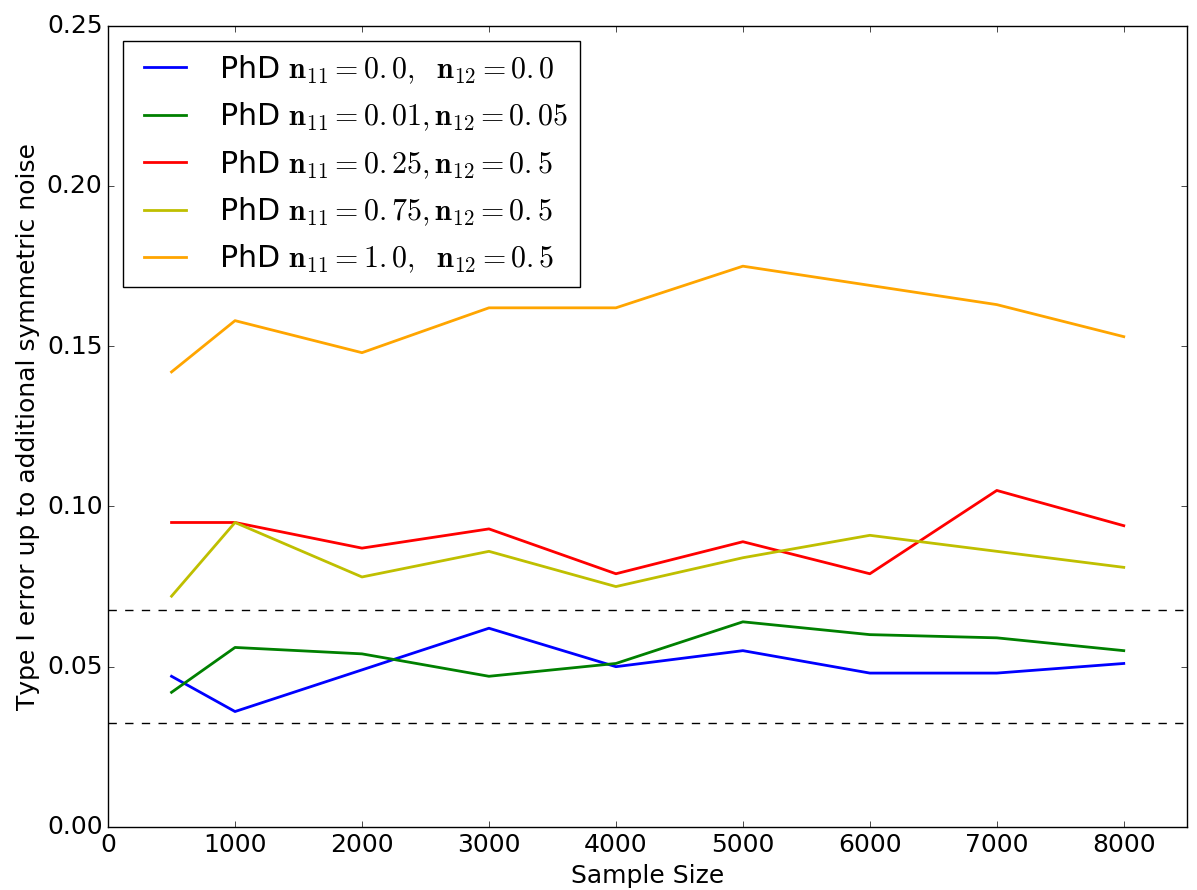}
    \end{subfigure}
    \caption{Extra Type I error results for the synthetic example with $\chi^2$
    \textbf{Left:} With no noise added for the ME, PhD and SME test. \textbf{Right:} Various additive Gaussian components, our base distribution without addition of noise is $\chi^2 (4)/ 4$. Here $n_{11}$ refers to the noise to signal ratio for the first set of samples and $n_{12}$ refers to the second set of samples.}
    \label{fig:phd}
\end{figure}
\vspace{0.3cm}
In figure above, the black dashed line is the $99\%$ Wald interval $\alpha \pm 2.57 \sqrt{\alpha (1 - \alpha) / 1000}$, where here $\alpha = 0.05$ is the significance level and $1000$ is the number of repetitions. 

On the left figure, we see that indeed all three test considered in this paper indeed controls the Type I error, when the underlying distribution between the two sets of sample is the same, note here no additional noise is added. 

On the right figure, we see that the PhD statistic controls Type I error for no added Gaussian noise, and also control Type I error for small differences in additive Gaussian components, unlike the ME test. However, we see that the type I error for a larger noise to signal ratio on the two set of samples indeed does alleviate the Type I error. This is not surprising, as the null distribution was constructed by using a permutation test, using:
\begin{equation*}
\varphi_{null} = \frac 1 2\varphi_{X_0} \varphi_U + \frac 1 2 \varphi_{X_0} \varphi_V = \varphi_{X_0}  (\frac 1 2 \varphi_{U} + \frac 1 2 \varphi_{V} ),
\end{equation*}
and if the estimated phase features are biased, in the regime with large additive Gaussian noise, then the following may not be true approximately: $\hat\rho_{null} = \hat\rho_{X_0} = \hat\rho_{Y_0}$, leading a to a biased null distribution.

In practice, if it is subtle effects we are looking for, with larger samples, we recommend the use of the SME test, however if this is not the case, then the PhD test is more appropriate, as it has good power for low sample size. In fact, the PhD test has power comparable with that of the ME test, however users should use it with caution, as it does not control the Type I error for larger additional SPD differences and requires more computational power.
\subsubsection{Higgs Dataset}
\FloatBarrier
\begin{table}[ht!]
\caption{Power for various sample size for high level features of the Higgs dataset}
\vskip 0.15in
\label{tab:higgs}
\begin{center}
\begin{small}
\begin{sc}
\begin{tabular}{lcccr}
\hline
Sample Size $N$ & SME Power & ME Power \\ 
\hline
$500$             & $0.94  $   &$ 1.0 $     \\
$600$             & $0.969  $   &$ 0.999 $   \\
$700$             & $0.987   $ & $1.0  $    \\
$800$             & $0.989  $   & $1.0  $    \\
$900$             & $0.994   $  & $ 1.0   $      \\
$1000$            &      $0.995$    & $1.0$          \\ 
\hline
\end{tabular}
\end{sc}
\end{small}
\end{center}
\vskip -0.1in
\end{table}
\FloatBarrier
The table here refers to the high level features of the Higgs dataset, which have been shown to be discriminative in \cite{baldi2014searching}. In this case, clearly both the ME and SME achieve good power, note here the SME has slightly less power, due to using only half of the samples to keep independence.
\FloatBarrier
\begin{figure}[ht!]
    \centering
    \begin{subfigure}
        \centering
        \includegraphics[width=0.49\textwidth]{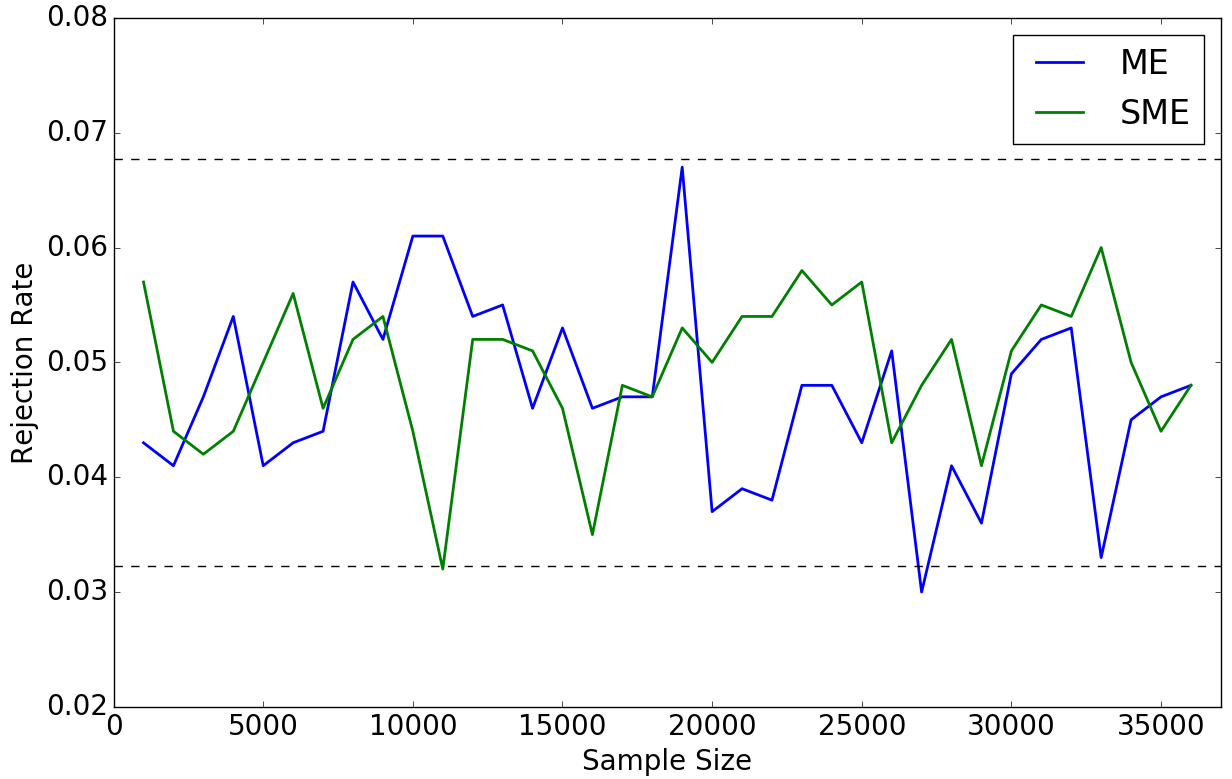}
    \end{subfigure}%
    ~ 
    \begin{subfigure}
        \centering
        \includegraphics[width=0.49\textwidth]{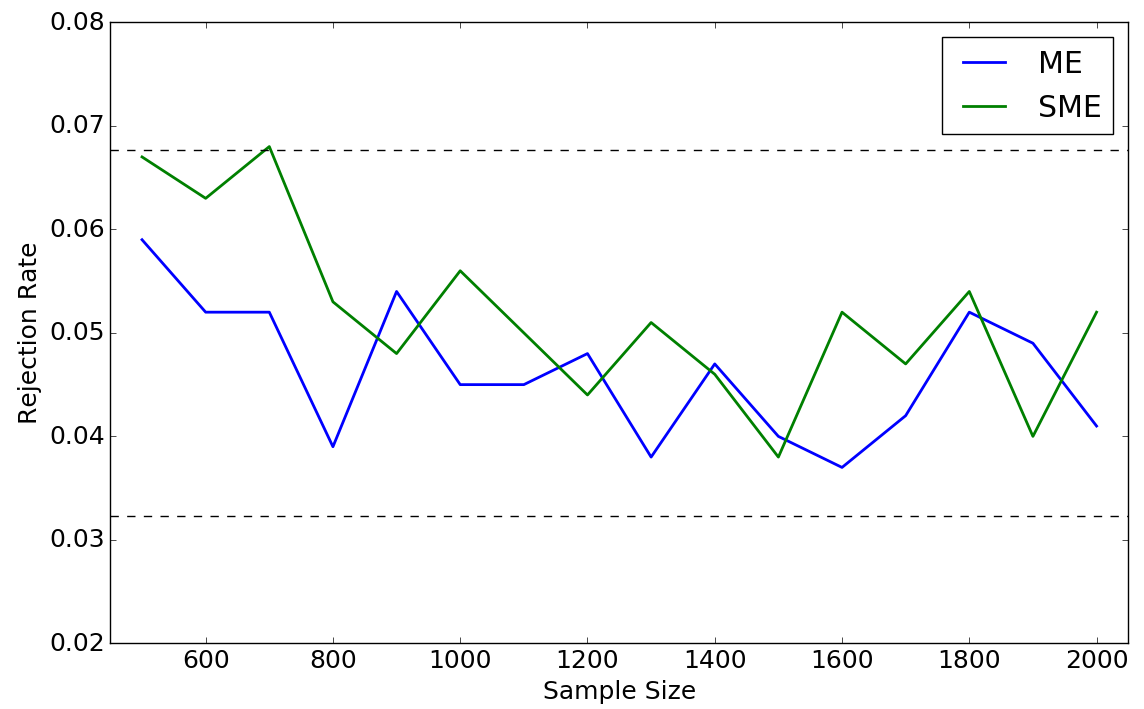}
    \end{subfigure}
    \caption{Type I error for the Higgs Dataset. \textbf{Left:} Extremely low level features \textbf{Right:} High level features. The black dashed line is the $99\%$ Wald interval $\alpha \pm 2.57 \sqrt{\alpha (1 - \alpha) / 1000}$, where here $\alpha = 0.05$ is the significance level and $1000$ is the number of repetitions.}
    \label{fig:higgs_type1}
\end{figure}
\vspace{0.3cm}
\FloatBarrier
The two figures here show that the Type I error is controlled for the ME and SME test, when we have $X_0 \eqd Y_0$, where we only consider samples drawn from $Y$, corresponding to the distribution of the processes where the Higgs Boson are produced. Note that on the right graph, the Type I error at first may be slightly alleviated due to small set of samples. 
\subsubsection{Aerosol Dataset}
\begin{figure}[ht!]
\centering
\includegraphics[scale = 0.4]{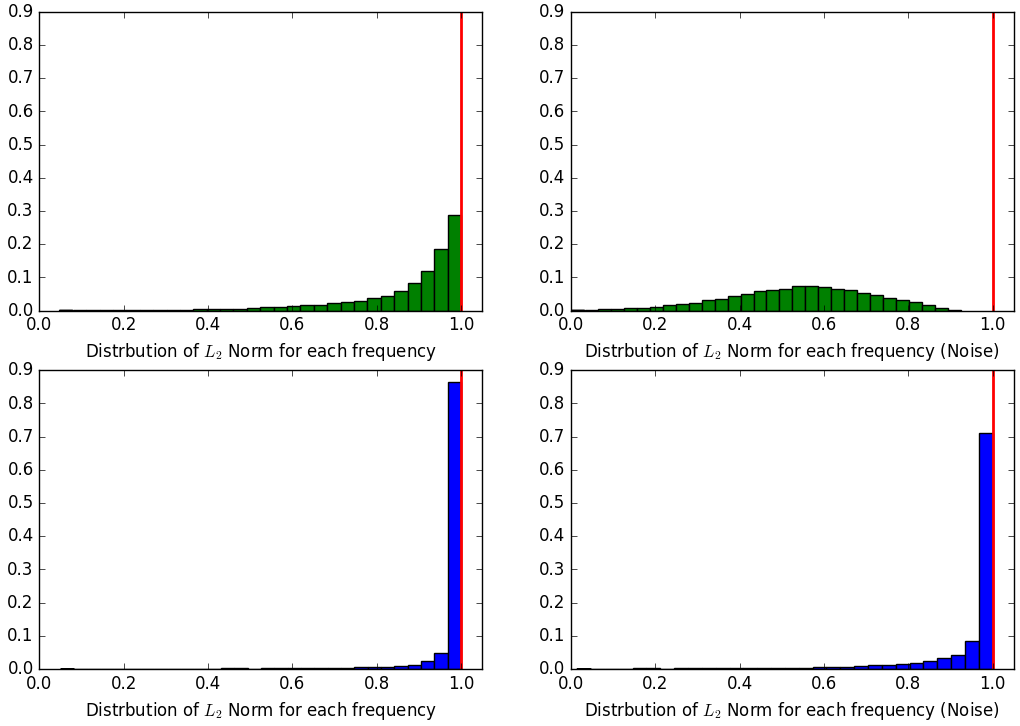}
\caption{Histograms for the distribution of the $L_2$ norm of the averages of Fourier features over each frequency $w$ for the original aerosol test set and the aerosol test set with added noise ($\sigma = 3$), here red line denotes the unit norm representing the phase features  \textbf{Top Green:}  Random Fourier Features $w$ (with the optimised kernel bandwidth)  \textbf{Bottom Blue:} Learnt Fourier features $w$ from the Fourier Neural Network.
\label{fig:freq_aero}
}
\end{figure}
\vspace{0.3cm}
We here provide some additional results for the Aerosol Dataset. First, we provide the average RMSE on the aerosol dataset (without noise on test set), based on $10$ runs, for different train and test splits in Table \ref{tab:aerosol}.
\FloatBarrier
\begin{table}[!ht]
\caption{Average RMSE for the Aerosol Dataset across $10$ runs, for different train and test splits, with standard deviation in brackets}
\vskip 0.15in
\label{tab:aerosol}
\begin{center}
\begin{small}
\begin{sc}
\begin{tabular}{lccccr}
\hline
 & Fourier NN & Phase NN & GLRR & PLRR \\ 
\hline
No noise & $0.101\ (0.011) $   &$ 0.101\ (0.008)$  &$ 0.079 \ (0.010) $ & $0.085\ (0.009)$ \\
\hline
\end{tabular}
\end{sc}
\end{small}
\end{center}
\vskip -0.1in
\end{table}
\FloatBarrier

In the experiments for the Aerosol covariate shift and above, we have seen that the Fourier NN performs similarly to the Phase NN, even under the addition of Gaussian noise, here we provide some possible insights. From the trained Fourier NN on the original dataset, we extract the frequencies $w$ learnt and compute $\left\Vert \mathbb{\hat{E}}\xi_{\omega}(X)\right\Vert$ for each frequency over the original and noisy test set, similarly we do this for the frequencies generated from the Gaussian kernel (with the optimised bandwidth on the original aerosol dataset). We show the empirical distribution of both of these in the figure above, we see that the discriminative frequencies learnt on the training data correspond to the Fourier features which are nearly normalised (i.e. they are close to unit norm like phase features, shown by the red line), this may suggest that the learnt frequencies have captured a notion of invariance to additive SPD components on just the training data. This provides insight into good performance of Fourier NN even under the covariate shift. It also indicates that the original Aerosol data potentially has irrelevant SPD noise components that the Fourier NN has learnt to ignore.

\section{Implementation Details}
\setcounter{figure}{0}    
\setcounter{table}{0}    
\setcounter{algorithm}{0}
\subsection{PhD two sample test}
For the PhD two sample test for the toy dataset, for each of the $1000$ runs, we use a permutation size of $400$, with the number of frequencies sampled set at $50$. Here the frequencies are sampled using the radial frequency distribution, where $\Sigma$ is chosen to be $\sigma^2 \textbf{I}$, with $\sigma^2$ being the empirical variance of the two set of samples.
The Radial Frequency Distribution is defined as follows:
$$ \mathbf{w} = R \Sigma ^ {- \frac{1}{2}} \boldsymbol{\psi} $$
where $\boldsymbol \psi \in \mathbb{R}^n$ is uniformly distributed on the $L_2$ unit sphere $\mathcal{S}_{n-1}$, and $R \in \mathbb{R}_{+}$ is a radius drawn independently from a folded Gaussian $\mathcal{N}^+(0,1)$. The radial frequency distribution is useful in high dimensions, as unlike the normal distributions, which `under samples' the low or middle frequencies, it is able to sample a broader range of frequencies due to its form. By covering a broader range of frequencies, we may be able to `better encode' information of the distribution represented by the bags, leading to a feature map that is more informative.

\subsection{Aerosol Dataset}
For the network, we use a squared loss function with an additional $L_2$ weight decay for regularisation, with a separate regularisation parameter for the two individual layers. For optimisation, we again use ADAM \cite{kingma2014adam} with fixed learning rate decay and $120$ epochs, with a batch size of $10$. We perform a 3-fold cross validation, and compute the MSE. We tune the learning rate, regularisation parameters and also number of frequencies for the neural network, here we initialise the first layer with Gaussian distribution with standard deviation = $1/\gamma_0$, where $\gamma_0$ denote the median heuristic.
\subsection{Dark Matter Dataset}
For all methods we sample frequencies from the normal distribution (with standard deviation = $1/\gamma_0$, where $\gamma_0$ denote the median heuristic.). After sampling a set of frequencies, we tune the scale of the set of frequencies and also the ridge regularisation parameter using the validation set. In particular we use $75$ frequencies on the first and second level of the kernel whenever they are used. Note we use the same set of frequencies (at each individual kernel level) across all the methods in a single run to allow for easier comparison, with potentially different scale tuned on the validation set. 
\end{document}